\documentclass{tlp}

\usepackage[T1]{fontenc}
\usepackage{amssymb}
\usepackage{amsmath}
\usepackage{epsfig}
\usepackage{theorem,multicol,algorithm}
\usepackage{url}

\DeclareMathSymbol{\naf}{\mathord}{symbols}{"18}

\newtheorem{theorem}{Theorem}[section] 
 
\newtheorem{definition}[theorem]{Definition}

\newtheorem{example}[theorem]{Example} 

\newcommand{\TRUE}{{\bf T}}
\newcommand{\FALSE}{{\bf F}}
\newcommand{\true}{{\bf t}}
\newcommand{\false}{{\bf f}}
\newcommand{\T}{\mathsf{true}}
\newcommand{\F}{\mathsf{false}}
\newcommand{\TRES}{\mathsf{T\mbox{\sf -}RES}}
\newcommand{\RES}{\mathsf{RES}}
\newcommand{\ERES}{\mathsf{E\mbox{\sf -}RES}}
\newcommand{\EASPT}{\mathsf{E\mbox{\sf -}ASP\mbox{\sf -}T}}
\newcommand{\ASPT}{\mathsf{ASP\mbox{\sf -}T}}

\newcommand{\union}{\cup}
\newcommand{\isect}{\cap}
\newcommand{\pair}[2]{\langle{#1},{#2}\rangle}
\newcommand{\GLred}[2]{{#1}^{#2}}
\newcommand{\IF}{\leftarrow}
\newcommand{\rsep}{.\;}
\newcommand{\lpeq}[1]{\equiv_{\mathrm{#1}}}
\newcommand{\bottom}[2]{\mathsf{bottom}(#1,#2)}
\newcommand{\rest}[3]{\mathsf{eval}(\topp{#1}{#2},#3)}
\newcommand{\topp}[2]{\mathsf{top}(#1,#2)}
\newcommand{\dep}[1]{\mathsf{Dep}^+\!(#1)}
\newcommand{\PHP}{{\rm PHP}_n^{n+1}}
\newcommand{\EXT}{{\rm EXT}^{l}}
\newcommand{\CPHP}{{\rm CPHP}_n^{n+1}}
\newcommand{\EPHP}{{\rm EPHP}_n^{n+1}}
\newcommand{\E}{{\rm E}}
\newcommand{\LOOP}{\mathsf{loop}}
\newcommand{\eb}{\mathsf{eb}}
\newcommand{\prog}{\Pi}

\newcommand{\visible}{\mathsf{v}}
\newcommand{\hidden}{\mathsf{h}}

\newcommand{\red}{\mathsf{red}}
\newcommand{\dlits}{\mathsf{dlit}}
\newcommand{\body}{\mathsf{body}}
\newcommand{\head}{\mathsf{head}}
\newcommand{\rules}{\mathsf{rule}}
\newcommand{\comp}{\mathsf{comp}}
\newcommand{\nlp}{\mathsf{nlp}}
\newcommand{\vars}{\mathsf{var}}

\newcommand{\clauset}{\mathcal{C}}

\newcommand{\atoms}{\mathsf{atom}}

\newcommand{\smodels}{\mathsf{smodels}}
\newcommand{\nomore}{\mathsf{noMore}}
\newcommand{\clasp}{\mathsf{clasp}}
\newcommand{\cmodels}{\mathsf{cmodels}}
\newcommand{\dlv}{\mathsf{dlv}}
\newcommand{\assat}{\mathsf{assat}}

\title[Theory and Practice of Logic Programming]{
Extended ASP Tableaux and\\Rule Redundancy in Normal Logic
Programs\thanks{This is an extended version of a paper 
\protect\cite{JarvisaloO:ICLP07} 
presented at the 23rd International Conference on
Logic Programming (ICLP 2007) in Porto, Portugal.}}

\author[M. J\"arvisalo and E. Oikarinen]
{MATTI J\"ARVISALO and EMILIA OIKARINEN \\
 Helsinki University of Technology (TKK) \\
 Department of Information and Computer Science\\
 P.O. Box 5400, FI-02015 TKK, Finland \\
 \email{matti.jarvisalo@tkk.fi, emilia.oikarinen@tkk.fi}}

\submitted{20 February 2008}
\revised{12 September 2008}
\accepted{18 September 2008}

\begin{document}
\maketitle


\begin{abstract}
We introduce an extended tableau calculus for answer set
programming~(ASP).  
The proof system is based on the ASP tableaux defined in
[Gebser\&Schaub, ICLP 2006], with an added extension rule.
We investigate the power of Extended ASP Tableaux both theoretically
and empirically.
We study  the relationship  of Extended ASP Tableaux with the
Extended Resolution proof system  defined by Tseitin for sets of clauses,
and separate Extended ASP Tableaux from ASP Tableaux by
giving a polynomial-length proof for a family of normal
logic programs $\{\prog_n\}$ 
for which ASP Tableaux has exponential-length minimal
proofs with respect to $n$.
Additionally, Extended ASP Tableaux imply interesting insight into
the effect of program simplification on the lengths of proofs in ASP.
Closely related to Extended ASP Tableaux, we empirically investigate
the effect of redundant rules on the efficiency of ASP solving.
\end{abstract}

\begin{keywords}
Answer set programming, tableau method, extension rule, proof
complexity, problem structure
\end{keywords}


\section{Introduction}

Answer set programming~(ASP)~\cite{%
MT99:slp,%
DBLP:journals/amai/Niemela99,%
DBLP:journals/ai/GelfondL02a,%
Lifschitz02,%
Baral:knowledge}
is a declarative problem solving paradigm which has proven successful
for a variety of knowledge representation and reasoning tasks
(see~\cite{DBLP:conf/asp/SoininenNTS01,%
DBLP:conf/padl/NogueiraBGWB01,%
DBLP:journals/tplp/ErdemLR06,%
Brooks:inferring} 
for examples).
The success has been brought forth by efficient solver
implementations such as
$\smodels$~\cite{DBLP:journals/ai/SimonsNS02},
$\dlv$~\cite{Leone:dlv}, 
$\nomore$++~\cite{DBLP:conf/lpar/AngerGLNS05},
$\cmodels$~\cite{GiunchigliaLM:answer},
$\assat$~\cite{Lin:assat}, and
$\clasp$~\cite{clasp}.
However, there has been an evident lack of theoretical 
studies into the reasons for the efficiency of  ASP solvers.

Solver implementations and their inference techniques can be seen as
deterministic implementations
of the underlying rule-based \emph{proof systems}. A
solver implements a particular proof system in the sense that the
propagation mechanisms applied by the solver apply the deterministic
deduction rules in the proof system, whereas 
the nondeterministic branching/splitting rule of the proof system
is made deterministic
through branching heuristics
present in typical solvers.
From the opposite point of view, a solver can be analyzed by investigating
the power of an abstraction of the solver as the proof system
the solver implements.
Due to this strong interplay between theory and practice, the study of
the relative efficiency of these proof systems reveals important new
viewpoints and explanations for the successes and failures of
particular solver techniques. 

A way of examining the \emph{best-case} performance of solver
algorithms is provided by \emph{(propositional) proof complexity
  theory}~\cite{Cook:relative,Beame:Propositional}, which concentrates
on studying the relative power of the proof systems underlying solver
algorithms in terms of the shortest existing proofs in the systems.
A large (superpolynomial) difference in the minimal length of proofs 
available in different proof systems for a family of Boolean
expressions reveals that solver implementations of these systems are
inherently different in strength.
While such proof complexity theoretic studies are frequent in the
closely related field of propositional satisfiability (SAT), where
typical solvers have been shown to be based on refinements of the
well-known Resolution proof system~\cite{Beame:understanding}, this
has not been the case for ASP.  
Especially, the inference techniques applied in current
state-of-the-art ASP solvers have been characterized by a family of
tableau-style ASP proof systems for normal logic programs only very
recently~\cite{DBLP:conf/iclp/GebserS06}, with some related proof
complexity theoretic investigations~\cite{DBLP:conf/ecai/AngerGJS06}
and generalizations~\cite{DBLP:conf/iclp/GebserS07}. 
The close relation of ASP and SAT and the respective theoretical
underpinning of practical solver techniques has also received little
attention up until
recently~\cite{DBLP:conf/iclp/GiunchigliaM05,gebsch06d}, 
although the fields could gain much by further studies on these
connections.

This work continues in part bridging the gap between ASP and SAT.
Influenced by Tseitin's \emph{Extended Resolution} proof
system~\cite{Tseitin:complexity} for clausal formulas, we introduce
\emph{Extended ASP Tableaux}, an extended tableau calculus based on
the proof system in~\cite{DBLP:conf/iclp/GebserS06}. 
The motivations for Extended ASP Tableaux are many-fold. 
Theoretically, Extended Resolution has proven to be among the most
powerful known proof systems, equivalent to, for example, extended Frege
systems; no exponential lower bounds for the lengths of proofs are
known for Extended Resolution.
We study the power of Extended ASP Tableaux, showing a tight
correspondence with Extended Resolution.

The contributions of this work are not only of theoretical nature.
Extended ASP Tableaux is in fact based on \emph{adding structure} into
programs by introducing additional \emph{redundant rules}.
On the practical level, the structure of problem instances 
has an important role in both ASP and SAT solving.
Typically, it is widely believed that redundancy can and should be
removed for practical efficiency.
However, the power of Extended ASP Tableaux reveals that this is not
generally the case, and such redundancy removing \emph{simplification}
mechanisms can drastically hinder efficiency.
In addition, we contribute by studying the effect of redundancy on the
efficiency of  a variety of ASP solvers.
The results show that the role of redundancy in programs is not as
simple as typically believed, and  controlled addition of redundancy
may in fact prove to be relevant in further strengthening the
robustness of current solver techniques.

The rest of this article is organized as follows. 
After preliminaries on ASP and SAT (Section~\ref{prel}), the
relationship of Resolution and ASP Tableaux proof systems and
concepts related to the complexity of proofs are discussed 
(Section~\ref{systems}).
By introducing the Extended ASP Tableaux proof system
(Section~\ref{easp}), proof complexity and simplification are then 
studied with respect to Extended ASP Tableaux (Section~\ref{comp}).
Experimental results related to Extended ASP Tableaux and redundant
rules in normal logic programs are presented in
Section~\ref{experiments}. 


\section{Preliminaries}
\label{prel}
As preliminaries we review basic concepts related to answer set
programming (ASP) in the context of  normal logic programs,
propositional satisfiability (SAT), and translations between ASP and
SAT. 
\subsection{Normal Logic Programs and Stable Models}
\label{programs}
We consider {\em normal logic programs} (NLPs) in the {\em
  propositional} case.
In the following we will review some standard concepts related to NLPs
and stable models.  

A normal logic program~$\prog$ consists of a finite set of rules of
the form  
\begin{equation}
\label{rule}
r \ : \ h\IF a_1,\ldots, a_n,\naf b_1,\ldots, \naf b_m,
\end{equation}
where each $a_i$ and $b_j$ is a propositional atom, and $h$ is either
a propositional atom, or the symbol $\bot$ that stands for falsity. 
A rule $r$ consists of a \emph{head}, $\head(r)=h$, and a \emph{body},
$\body(r) = \{a_1,\ldots,a_n, \naf b_1, \ldots,\naf b_m\}$.
The symbol~``$\naf$'' denotes {\em default negation}. 
A {\em default literal} is an atom $a$, or its default negation
$\naf a$. 

The set of atoms occurring in a program~$\prog$ is $\atoms(\prog)$, and
$$\dlits(\prog) = \{a,\naf a \mid a \in \atoms(\prog)\}$$
 is the set of default literals in $\prog$. 
We use the shorthands $L^+=\{a\mid a\in L\}$ and  $L^-=\{a\mid \naf
a\in L\}$ for a set~$L$ of default literals, and $\naf A=\{\naf a\mid
a\in A\}$ for a set~$A$ of atoms.
This allows the shorthand 
$$\head(r) \IF \body(r)^+ \cup \naf \body(r)^-$$ for~(\ref{rule}). 
A rule $r$ is a \emph{fact} if $\body(r)=\emptyset$.
Furthermore, we use the shorthands
\begin{eqnarray*}
\head(\prog) &=& \{\head(r)\mid r\in\prog\} \mbox{ and}\\
\body(\prog) &=& \{\body(r)\mid r\in\prog\}.
\end{eqnarray*}
 
In ASP, we are interested in \emph{stable models}~\cite{GL88:iclp} (or
\emph{answer sets}) of a program $\prog$. 
An \emph{interpretation} $M \subseteq \atoms(\prog)$ 
defines which atoms of $\prog$ are true ($a\in M$) and
which are false ($a\not\in M$).
An interpretation $M\subseteq\atoms({\prog})$ is a \emph{(classical)
model} of $\prog$ if and only if $\body(r)^+\subseteq M$ and
$\body(r)^- \isect M=\emptyset$ imply $\head(r)\in M$ for each rule $r
\in \prog$. 
A model $M$ of a program $\prog$ is a stable model of~$\prog$
if and only if there is no model $M' \subset M$ of $\GLred{\prog}{M},$
where 
$$\GLred{\prog}{M} = \{ \head(r)\IF
\body(r)^+\mid r \in \prog {\rm \; and \; } \body(r)^-\isect M=\emptyset\}$$
is called the \emph{Gelfond-Lifschitz reduct} of $\prog$ with respect
to $M$. 
We say that a program~$\prog$ is \emph{satisfiable} if it has a stable
model, and \emph{unsatisfiable} otherwise.

The \emph{positive dependency graph} of $\prog$, denoted by
$\dep{\prog}$, is a directed graph with $\atoms({\prog})$ and
$$\{\langle b,a \rangle \mid 
\exists r\in\prog\mbox{ such that }b=\head(r)\mbox{ and }a\in
\body(r)^+\}$$ as the sets of vertices and edges, respectively. 
A non-empty set $L\subseteq\atoms(\prog)$ is a loop in $\dep{\prog}$
if for any $a,b\in L$ there is a path of non-zero length from $a$ to
$b$ in $\dep{\prog}$ such that all vertices in the path are in $L$.
We denote by $\LOOP(\prog)$  the set of all loops in $\dep{\prog}$.
A NLP is \emph{tight} if and only if $\LOOP(\prog)=\emptyset$. 
Furthermore,  the \emph{external bodies} of a set $A$ of atoms 
in $\prog$ is 
$$\eb(A) = \{\body(r) \mid r\in \prog,\; \head(r) \in A,\;
\body(r)^+ \cap A =\emptyset\}.$$ 
A set $U\subseteq\atoms(\prog)$ is \emph{unfounded} if
$\eb(U)=\emptyset$. 
We denote the {\em greatest unfounded set}, that is, the union of all 
unfounded sets, of $\prog$ by $\mathsf{gus}(\prog)$.

A {\em splitting set}~\cite{LT94:iclp} for a NLP $\prog$ is any 
set $U\subseteq \atoms(\prog)$ such that for every $r\in\prog$, if 
$\head(r)\in U$, then $\body(r)^+\union \body(r)^-\subseteq U$.
The {\em bottom} of $\prog$ relative to $U$ is 
$$\bottom{\prog}{U}= \{r\in \prog\mid \atoms(\{r\})\subseteq U\},$$
and the {\em top} of $\prog$ relative to $U$ is
$$\topp{\prog}{U}=\prog\setminus\bottom{\prog}{U}.$$  
The top can be partially evaluated with respect to an interpretation 
$X\subseteq U$. The result is a program $\rest{\prog}{U}{X}$ that
contains the rule
$$\head(r)\IF(\body(r)^+\setminus U), \naf(\body(r)^-\setminus U)$$
for each $r\in \topp{\prog}{U}$ such that $\body(r)^+\isect U\subseteq
X$ and $(\body(r)^-\isect U)\isect X =\emptyset$.
Given a splitting set $U$ for a NLP $\prog$, a {\em solution} to $\prog$
with respect to $U$ is a pair~$\pair{X}{Y}$ such that $X\subseteq U$,
$Y\subseteq\atoms(\prog)\setminus U$, $X$ is a stable model of
$\bottom{\prog}{U}$, and $Y$ is a stable model of $\rest{\prog}{U}{X}$.
In this work we will apply the 
{\em splitting set theorem}~\cite{LT94:iclp} 
that relates solutions with stable models. 
\begin{theorem}[\cite{LT94:iclp}]
\label{splitting}
Given a normal logic program $\prog$ and a splitting set $U$ for
$\prog$, 
an interpretation $M\subseteq\atoms(\prog)$ is a stable model of
$\prog$ if and only if $\pair{M\isect U}{M\setminus U}$ is a solution
to $\prog$ with respect to $U$. 
\end{theorem}

\subsection{Propositional Satisfiability}
Let $X$ be a set of Boolean variables. Associated with every variable
$x\in X$ there are two \emph{literals}, the positive literal, denoted
by $x$, and the negative literal, denoted by~$\bar x$. 
A \emph{clause} is a disjunction of distinct literals. 
We adopt the standard convention of viewing a clause as a finite set of
literals and a \emph{CNF formula} as a finite set of clauses. 
The set of variables appearing in a clause $C$ (a set $\clauset$ of
clauses, respectively) is denoted by
$\vars(C)$ ($\vars(\clauset)$, respectively).

A \emph{truth assignment} $\tau$ associates a truth value
$\tau(x)\in\{\F,\T\}$ with each variable $x\in X$. A truth assignment
\emph{satisfies} a set of clauses if and only if it satisfies every  
clause in it. A clause is satisfied if and only if it contains at
least one satisfied literal, where a literal~$x$ ($\bar x$,
respectively) 
is satisfied if $\tau(x)=\T$ ($\tau(x)=\F$, respectively). 
A set of clauses is \emph{satisfiable} if there is a truth assignment that
satisfies it, and \emph{unsatisfiable} otherwise.
\subsection{SAT as ASP}
There is a natural linear-size translation from sets of clauses to
normal logic programs so that the stable models of the encoding
represent the satisfying truth assignments of the original set of clauses
\emph{faithfully}, that is,
there is a bijective correspondence between the satisfying truth
assignments and stable models of the
translation~\cite{DBLP:journals/amai/Niemela99}. 
Given a set~$\clauset$ of clauses, this translation $\nlp(\clauset)$
introduces a new atom $c$ for each clause $C \in \clauset$, and
atoms~$a_x$ and $\hat a_x$ for each variable~$x \in \vars(\clauset)$. 
The resulting NLP is then 
\begin{eqnarray}
\nlp(\clauset) & = & 
\{a_x \IF \naf \hat a_x\rsep\;\hat a_x \IF \naf a_x\mid
x \in \vars(\clauset) \} \cup \label{nlp1}\\
&&
\{\bot \IF \naf c \mid {C \in \clauset}\}
\cup \label{nlp2}\\
&& 
\{c \IF a_x \mid x\in C,\ C \in \clauset,\ x\in\vars(C)\} \cup
\label{nlp3}\\
&& 
\{c \IF \naf a_x \mid \bar x \in C,\ C \in \clauset,\ 
x\in\vars(C)\}.\label{nlp4}  
\end{eqnarray}
The rules (\ref{nlp1}) encode  that 
each variable must be assigned an unambiguous truth value,
the rules in (\ref{nlp2}) that 
each clause in $\clauset$ must be satisfied,
while (\ref{nlp3}) and (\ref{nlp4}) encode that each clause is satisfied if
at least one of its literals is satisfied.

\begin{example}
\label{ex:clauses2nlp}
The set 
$\clauset  = \{\{x, y\},\{ x,\bar y\}, \{\bar x,  y\}, \{\bar x, \bar
y\} \}$ of clauses is represented by the normal logic program 
\begin{eqnarray*}
\nlp(\clauset) & = &\{\;a_x\IF \naf \hat a_x\rsep \hat a_x\IF \naf a_x\rsep
a_y\IF \naf \hat a_y\rsep\hat a_y\IF \naf a_y\rsep
\\ && \;\;\bot\IF\naf c_1\rsep\bot\IF\naf c_2 \rsep
\bot\IF\naf c_3 \rsep\bot\IF\naf c_4 \rsep \\ &&
\;\;c_1\IF a_x\rsep c_1\IF a_y\rsep
c_2\IF a_x\rsep c_2\IF \naf a_y\rsep \\ &&
\;\;c_3\IF \naf a_x\rsep c_3\IF a_y\rsep
c_4\IF \naf a_x\rsep c_4\IF \naf a_y\;\}.
\end{eqnarray*}
\end{example}

\subsection{ASP as SAT}
Contrarily to the case of translating SAT into ASP,
there is no modular\footnote{Intuitively, for a modular translation, 
adding a set of facts to a program leads to a local change not
  involving the translation of the rest of the
  program~\cite{DBLP:journals/amai/Niemela99}.} 
and faithful translation from normal logic programs to propositional 
logic~\cite{DBLP:journals/amai/Niemela99}.
Moreover, any faithful translation is potentially of exponential size
when additional variables are not
allowed~\cite{DBLP:journals/tocl/LifschitzR06}\footnote{However,
polynomial-size propositional encodings using extra variables are  
known, see~\cite{DBLP:journals/amai/Ben-EliyahuD94,%
DBLP:conf/ijcai/LinZ03,%
DBLP:journals/jancl/Janhunen06}.
Also, ASP as Propositional Satisfiability approaches for solving normal logic
programs have been developed, for example, $\assat$~\cite{Lin:assat}
 (based on incrementally adding---possibly exponentially many---loop
formulas) and \textsf{asp-sat}~\cite{GiunchigliaLM:answer} 
(based on generating a \emph{supported model}~\cite{DBLP:conf/lpnmr/BrassD95} 
of the program and testing its minimality---thus avoiding exponential
space consumption).  
}.
However, for any tight program $\prog$ it holds that the answer sets
of $\prog$ can be characterized faithfully by the 
satisfying truth assignments
of a~linear-size propositional formula called \emph{Clark's
completion}~\cite{Clark:negation,Fages94} 
of $\prog$, defined using a Boolean variable~$x_a$ for each $a \in
\atoms(\prog)$ as 
\begin{eqnarray}
C(\prog) & = &   \bigwedge_{h \in \atoms(\prog)\union\{\bot\}}
\Bigg (x_h \leftrightarrow
\bigvee_{r \in \rules(h)} \Bigg ( \bigwedge_{b \in \body(r)^+} x_b  \wedge
\bigwedge_{b \in \body(r)^-} \bar x_b \Bigg ) \Bigg), \label{clarks}
\end{eqnarray}
where $\rules(h) = \{r\in \prog \mid \head(r)=h\}$.
Notice that there are the special cases
(i)~if~$h$ is $\bot$ then the equivalence becomes the negation of the right
hand side, (ii)~if $h$ is a fact, then the equivalence reduces to
the clause $\{x_h\}$, and (iii)~if an atom $h$ does not appear in the head 
of any rule then the equivalence
reduces to the clause~$\{\bar x_h\}$.

In this work, we will consider the clausal
representation of Boolean formulas.
A~linear-size clausal translation of $C(\prog)$ is achieved by introducing
additionally a new Boolean variable $x_B$ for each $B \in
\body(\prog)$.
Using the new variables for the bodies, we arrive at the \emph{clausal
  completion} %
\begin{eqnarray}
\comp(\prog) &\!\!\!=\!\! &  
\bigcup_{B \in \body(\prog)}    
\Bigg \{x_B\equiv\bigwedge_{a \in B^+}x_a \wedge \bigwedge_{b \in
  B^-} \bar x_b\Bigg \} \cup 
\bigcup_{B \in \body(\rules(\bot))} \{\{\bar x_B\}\}
\label{comp1}\\
&& \cup
\bigcup_{
h \in \head(\prog) \setminus \{\bot\}}
\Bigg \{x_h \equiv \bigvee_{B \in \body(\rules(h)) } x_B \Bigg \} 
\label{comp2}\\ 
&& \cup
\bigcup_{a \in \atoms(\prog) \setminus \head(\prog)} \{\{\bar x_a\}\},
\label{comp3} 
\end{eqnarray}
where the shorthands $x \equiv \bigwedge_{x_i \in X} x_i$ and 
$x \equiv \bigvee_{x_i \in X} x_i$ stand for the sets of clauses 
$\{x, \bar x_1, \ldots, \bar x_n\} \cup \bigcup_{x_i\in X} \{\bar x, x_i\}$
and
$
\bigcup_{x_i\in X} \{x, \bar x_i\} \cup 
\{\bar x, x_1, \ldots, x_n\}$, 
respectively.

\begin{example}
\label{ex:completion}
For the normal logic program 
$\Pi = \{a\IF b, \naf a.\; b\IF c.\; c\IF\naf b\}$, the
 clausal completion is
\begin{eqnarray*}
\comp(\Pi) & = &\{
\{x_{\{b, \naf a\}},x_a,\bar{x}_b\},
\{\bar{x}_{\{b, \naf a\}},\bar{x}_a\},
\{\bar{x}_{\{b,\naf a\}},x_b\},\\ &&
\;\;
\{x_{\{c\}}, \bar{x}_c\},
\{\bar{x}_{\{c\}}, x_c\},
\{x_{\{\naf b\}}, x_b\},
\{\bar{x}_{\{\naf b\}},\bar{x}_b\},
\{x_a, \bar{x}_{\{b, \naf a\}}\},\\ &&
\;\;\{\bar{x}_a, x_{\{b,\naf a\}}\},
\{x_b,\bar{x}_{\{c\}}\},
\{\bar{x}_b,x_{\{c\}}\},
\{x_c,\bar{x}_{\{\naf b\}}\}\},
\{\bar{x}_c,x_{\{\naf b\}}\}.
\end{eqnarray*}
\end{example}


\section{Proof Systems for ASP and SAT}
\label{systems}

In this section we review concepts related to proof
complexity~\cite{Cook:relative,Beame:Propositional}
in the context of this work, and discuss  the
relationship of Resolution and ASP Tableaux~\cite{DBLP:conf/iclp/GebserS06}.
\subsection{Propositional Proof Systems and  Complexity}
Formally, a \emph{(propositional) proof system} is a polynomial-time
computable predicate~$S$ such that a propositional expression~$E$ is
unsatisfiable if and only if there is a \emph{proof}~$P$ for
which~$S(E,P)$ holds.  
A proof system is thus a polynomial-time procedure for checking the
correctness of proofs in a certain format. 
While proof checking is efficient, finding short proofs may be
difficult, or, generally, impossible since short proofs may not exist 
for a too weak proof system.
As a measure of hardness of proving  unsatisfiability of an
expression~$E$ in a proof system~$S$, the \emph{(proof) complexity}
of~$E$ in~$S$ is the length of the \emph{shortest} proof for~$E$
in~$S$.
For a family~$\{E_n\}$ of unsatisfiable expressions over increasing
number of variables, the (asymptotic) complexity of~$\{E_n\}$ is
measured with respect to the sizes of $E_n$. 

For two proof systems $S$ and $S'$, we say that $S'$ 
\emph{polynomially  simulates} $S$
if for all families $\{E_n\}$ it holds that $C_{S'}(E_n) \le
p(C_S(E_n))$ for all $E_n$,  where $p$ is a polynomial, and $C_S$ and
$C_{S'}$ are the complexities in $S$ and $S'$, respectively.
If $S$ simulates $S'$ and vice versa, then $S$ and $S'$ are
\emph{polynomially equivalent}.
If there is a family $\{E_n\}$ for which $S'$ does not
polynomially simulate $S$, we say that $\{E_n\}$ \emph{separates} $S$
from $S'$. If $S$ simulates $S'$, and there is a family $\{E_n\}$
separating $S$ from $S'$, then $S$ is \emph{more powerful} than $S'$.

\subsection{Resolution}

The well-known Resolution proof system ($\RES$) for sets of clauses
is based on the \emph{resolution rule}. 
Let $C,D$ be clauses, and $x$ a Boolean variable.
The resolution rule states that we can \emph{directly derive}  $C \cup
D$ from  $\{x\} \cup C$ and $\{\bar x\} \cup D$ by \emph{resolving
on}~$x$.  

A \emph{$\RES$ derivation} of a clause $C$ from a set $\clauset$ of
clauses is a sequence of clauses $\pi = (C_1, C_2, \ldots, C_n)$,
where $C_n = C$ and each $C_i$, where $1 \le i < n$, is either
(i)~a~clause in $\clauset$ (an \emph{initial clause}),  
or (ii)~derived with the resolution rule from two clauses $C_j,C_k$,
where $j,k < i$ (a \emph{derived clause}). 
The \emph{length} of $\pi$ is $n$, the number of clauses occurring in
it. 
Any derivation of the empty clause $\emptyset$ from $\clauset$ is a
\emph{$\RES$ proof} for (the unsatisfiability of) $\clauset$.

Any $\RES$ proof $\pi = (C_1, C_2, \ldots, C_n=\emptyset)$ 
can be represented as a directed acyclic graph, in which
the leafs are initial clauses and other nodes  are derived clauses. 
There are edges from $C_i$ and $C_j$ to $C_k$ if and only if~$C_k$ has
been directly derived from $C_i$ and $C_j$ using the resolution rule. 
Many \emph{Resolution refinements}, in which the structure of the
graph representation  is restricted, have been proposed and studied. 
Of particular interest here is
\emph{Tree-like Resolution} ($\TRES$),
in which it is required that proofs are represented by trees.
This implies that a derived clause, if subsequently used multiple
times in the proof, must be derived anew each time from initial clauses.

$\TRES$ is a \emph{proper} $\RES$ refinement, that is, $\RES$ is more
powerful than $\TRES$~\cite{Bensasson:near}.
On the other hand, it is well known that the DPLL
method~\cite{Davis:Computing,Davis:Machine}, the basis of most
state-of-the-art SAT solvers, is polynomially equivalent to $\TRES$.
However, conflict-learning DPLL is more powerful than $\TRES$,  and
polynomially equivalent to $\RES$ under a slight
generalization~\cite{Beame:understanding}. 
\subsection{ASP Tableaux}
Although ASP solvers for normal logic programs have been available
for many years, the deduction rules applied in such  solvers
have only recently been formally defined as a proof system, which we
will here refer to as ASP Tableaux or
$\ASPT$~\cite{DBLP:conf/iclp/GebserS06}. 

An ASP tableau for a NLP $\prog$ is a binary tree of the following
structure. 
The \emph{root} of the tableau consists of the rules $\prog$ and the
\emph{entry} $\FALSE \bot$ for capturing that $\bot$ is always false.
The non-root nodes of the tableau are single \emph{entries} of the form
$\TRUE a$ or~$\FALSE a$, where $a \in \atoms(\prog) \cup
\body(\prog)$.
As typical for tableau methods, entries are generated by
\emph{extending} a \emph{branch} (a path from the root to a leaf node) by
applying one of the rules in Figure~\ref{rules}; if the prerequisites of
a rule hold in a branch, the branch can be extended with the entries
specified by the rule.
For convenience, we use shorthands $\true l$ and $\false l$ for
default literals:
\begin{eqnarray*}
\true l&=&\left\{\begin{array}{@{}l}
\TRUE a, \mbox{ if } l=a\mbox{ is positive,}\\
\FALSE a, \mbox{ if } l=\naf a\mbox{ is negative; and}
  \end{array}\right. \\
\false l&=&\left\{\begin{array}{@{}l}
\TRUE a, \mbox{ if } l=\naf a\mbox{ is negative,}\\
\FALSE a, \mbox{ if } l=a\mbox{ is positive.}
  \end{array}\right.
\end{eqnarray*}

\begin{figure}[!h]
\begin{center}
\begin{tabular}{c}
\begin{minipage}{1.6cm}
$\begin{array}{c|c}
\hline
\TRUE \phi & \FALSE \phi
\end{array}$
\end{minipage}
\begin{minipage}{.5cm}
\quad \quad \quad \quad \quad ($\natural$)
\medskip \ \\
\end{minipage}
 \\
(a) Cut 
\end{tabular}

\ \\
\ \\

\begin{tabular}{c@{}c}
\begin{minipage}{5.5cm}
\centering
$
\begin{array}{c}
h \IF l_1,\ldots,l_n\\
\true l_1,\ldots,\true l_n\\
\hline
\TRUE\{l_1,\ldots,l_n\}
\end{array}
$
\\
(b) Forward True Body
\end{minipage}
&
\begin{minipage}{5.5cm}
\centering
$
\begin{array}{c}
\FALSE\{l_1,\ldots,l_i,\ldots,l_n\} \\
\true l_1,\ldots,\true l_{i-1},\true l_{i+1},\ldots,\true l_n\\
\hline
\false l_i
\end{array}
$
\\
(c) Backward False Body
\end{minipage}
\end{tabular}

\ \\
\ \\

\begin{tabular}{c@{}c}
\begin{minipage}{5.5cm}
\centering
$
\begin{array}{c}
h \IF l_1,\ldots,l_n\\
\TRUE\{l_1,\ldots,l_n\}\\
\hline
\TRUE h
\end{array}
$
\\
(d) Forward True Atom
\end{minipage}
&
\begin{minipage}{5.5cm}
\centering
$
\begin{array}{c}
h \IF l_1,\ldots,l_n\\
\FALSE h\\
\hline
\FALSE\{l_1,\ldots,l_n\}
\end{array}
$
\\
(e) Backward False Atom
\end{minipage}
\end{tabular}

\ \\
\ \\

\begin{tabular}{c@{}c}
\begin{minipage}{5.5cm}
\centering
$
\begin{array}{c}
h \IF l_1,\ldots,l_i,\ldots,l_n\\
\false l_i\\
\hline
\FALSE\{l_1,\ldots,l_i,\ldots,l_n\}
\end{array}
$
\\
(f) Forward False Body
\end{minipage}
&
\begin{minipage}{5.5cm}
\centering
$
\begin{array}{c}
\ \\
\TRUE\{l_1,\ldots,l_i,\ldots,l_n\}\\
\hline
\true l_i
\end{array}
$
\\
(g) Backward True Body
\end{minipage}
\end{tabular}

\ \\
\ \\
\ \\

\begin{tabular}{c@{}c}
\begin{minipage}{6cm}
\centering
\begin{minipage}{1.5cm}
$
\begin{array}{c}
\FALSE B_1,\ldots,\FALSE B_m\\
\hline
\FALSE h
\end{array}
$
\end{minipage}
\quad \quad \quad ($\flat$)
\\
\centering
(h) \ \ \ 
\end{minipage}
&
\begin{minipage}{6.5cm}
\begin{minipage}{4.5cm}
\centering
\vspace{-0.3cm}
$
\begin{array}{c}
\FALSE B_1,\ldots,\FALSE B_{i-1},\FALSE B_{i+1},\ldots,\FALSE B_m\\
\TRUE h\\
\hline
\TRUE B_i
\end{array}
$
\end{minipage}
\quad \quad  ($\sharp$) 
\\
\centering
(i) \ \ \ \ 
\end{minipage} \\
\end{tabular}

\vspace{1ex}
\begin{tabular}{l@{\hspace{-0.5ex}}l}
($\natural$): & Applicable when $\phi \in \atoms(\prog) \cup \body(\prog)$.\\
($\flat$): & Applicable when one of the following conditions holds: \\
&$\S$ (``Forward False Atom''), $\dag$ (``Well-Founded Negation''), or 
$\ddag$ (``Forward Loop'').\\
($\sharp$): & Applicable when one of the following conditions holds: \\
&$\S$ (``Backward True Atom''), $\dag$ (``Well-Founded Justification''), or 
$\ddag$ (``Backward Loop'').\\
($\S$): & Applicable when $\body(\rules(h)) = \{B_1,\ldots,B_m\}$. \\
($\dag$):  &  Applicable when\\
&$\{B_1,\ldots,B_m\}\subseteq\body(\prog)$ and $h\in
\mathsf{gus}\;\!(\{r\in\prog\mid\body(r)\not\in\{B_1,\ldots
B_m\}\})$.\\ 
($\ddag$):& Applicable when $h \in L$, $L \in \LOOP(\prog)$, and
$\eb(L) = \{B_1, \ldots, B_m\}$ all hold.
\end{tabular}
\caption{Rules in ASP Tableaux.}
\label{rules}
\end{center}
\end{figure}

A branch is \emph{closed under} the \emph{deduction rules} (b)--(i) if
the branch cannot be extended
using the rules.
A branch is \emph{contradictory} if there are the entries
$\TRUE a$ and~$\FALSE a$ for some $a$.
A branch is \emph{complete} if it is contradictory, or if there is the
entry $\TRUE a$ or $\FALSE a$ for each $a \in \atoms(\prog) \cup
\body(\prog)$ and the branch is closed
under the deduction rules (b)--(i). 
A tableau is contradictory, if all its branches in are
contradictory, and \emph{non-contradictory} otherwise.
A tableau is complete if all its branches are complete. 
A contradictory tableau from $\prog$ is an $\ASPT$ proof for (the
unsatisfiability of) $\prog$.  
The \emph{length} of an $\ASPT$ proof is the number of entries in it.

\begin{example}
An $\ASPT$ proof for the NLP
$\Pi = \{a\IF b, \naf a.\; b\IF c.\; c\IF\naf b\}$ is shown in 
Figure~\ref{tab-ex}, with the rule applied for
deducing each entry given in parentheses.
For example, the entry $\FALSE a$ has been deduced from 
$a\IF b, \naf a$ in $\prog$ and the entry $\TRUE\{b, \naf a\}$ in the
left branch by applying the rule (g) Backward True Body.
On the other hand, $\TRUE\{b, \naf a\}$ has been deduced from
$a\IF b, \naf a$ in $\prog$ and the entry~$\TRUE a$ 
in the left branch by applying the rule (i$\S$), that is,
rule (i) by the fact that the condition $\S$ ``Backward True Atom'' is
fulfilled (in $\prog$, the only body with atom~$a$ in the head is
$\{b, \naf a\}$). 
The tableau in Figure~\ref{tab-ex} has two closed branches: 
$$(\prog \cup \{\FALSE \bot\}, \TRUE a, \TRUE\{b, \naf a\}, \FALSE a)
\mbox{ and }$$
$$(\prog \cup \{\FALSE \bot\}, \FALSE a, \FALSE\{b,\naf a\}, \FALSE b,
\TRUE\{\naf b\}, \TRUE c, \TRUE\{c\}, \TRUE b).$$
These branches share the common prefix $(\prog \cup \{\FALSE \bot\})$.
\begin{figure}[!h]
\input{asp-tableu-proof.pstex_t}%
\caption{An
  $\ASPT$ proof for $\Pi = \{a\IF b,\naf a.\; b\IF c.\; c\IF\naf b\}$.}
\label{tab-ex}
\end{figure}
\end{example}

Any branch $B$ describes a \emph{partial assignment}
$\mathcal{A}$ on $\atoms(\prog)\union\body(\prog)$ 
in a natural way, that is, if there is an entry $\TRUE a$ ($\FALSE a$,
respectively)  in $B$ for $a\in \atoms(\prog)\union\body(\prog)$, 
then $(a,\T) \in \mathcal{A}$ ($(a,\F) \in
\mathcal{A}$, respectively).
$\ASPT$ is a sound and complete proof system
for normal logic programs, that is, there is a complete
non-contradictory ASP tableau from $\prog$ if and only if $\prog$ is
satisfiable~\cite{DBLP:conf/iclp/GebserS06}. 
Thus the assignment $\mathcal{A}$ described by 
a complete non-contradictory branch gives a stable model 
\mbox{$M = \{a \in \atoms(\prog) \mid (a,\T) \in \mathcal{A}\}$} of $\prog$.

As argued in~\cite{DBLP:conf/iclp/GebserS06}, current ASP solver
implementations are tightly related to $\ASPT$, with the
intuition that the cut rule is made deterministic with decision heuristics,
while the deduction rules describe the propagation mechanism in ASP solvers.
For instance, the
$\nomore$++ system~\cite{DBLP:conf/lpar/AngerGLNS05} 
is a deterministic implementation of the rules
(a)--(g),(h$\S$),(h$\dag$), and~(i$\S$), while
$\smodels$~\cite{DBLP:journals/ai/SimonsNS02} applies
 the same rules with the cut rule 
restricted to~$\atoms(\prog)$.

Interestingly, $\ASPT$ and $\TRES$ are polynomially
equivalent under the translations $\comp$ and $\nlp$.
Although the similarity of unit propagation in DPLL and propagation
in ASP solvers is discussed 
in~\cite{DBLP:conf/iclp/GiunchigliaM05,gebsch06d}, here we want to stress
the direct connection between $\ASPT$ and $\TRES$.
In detail, $\TRES$ and $\ASPT$ are equivalent in the sense that 
(i) given an arbitrary NLP $\prog$, the length of minimal $\TRES$ proofs 
for $\comp(\prog)$ is polynomially bounded in the 
the length of  minimal $\ASPT$ proofs for $\prog$, and
(ii) given an arbitrary set $\clauset$ of clauses, the length of
minimal $\ASPT$ proofs for $\nlp(\clauset)$ is polynomially bounded in
the length of minimal $\TRES$ proofs for $\clauset$.

\begin{theorem} 
\label{tres-asp}
$\TRES$ and $\ASPT$ are polynomially equivalent proof systems 
in the sense that
\begin{itemize}
\item[(i)]
considering tight normal logic programs,
$\TRES$ under the translation $\comp$ polynomially
simulates $\ASPT$, and
\item[(ii)]
considering sets of clauses,
$\ASPT$ under the translation $\nlp$ polynomially simulates
$\TRES$.
\end{itemize}
\end{theorem}

In the following we give detailed proofs for the two parts of Theorem 
\ref{tres-asp} followed by illustrating examples.

In the proof of the first part of Theorem \ref{tres-asp}, we use
a concept of a \emph{(binary) cut tree} corresponding to an $\ASPT$
proof.   
Given an $\ASPT$ proof $T$ for a normal logic program $\prog$, 
the corresponding cut tree is obtained as follows.
Starting from the root of $T$, we replace each non-leaf entry
generated by a deduction rule in $T$ by an application of the 
cut rule on the corresponding entry.
For example,
the cut tree~$T'$ corresponding to the $\ASPT$ proof~$T$ in
Figure~\ref{tab-ex} is given in Figure~\ref{tres-ex}~(left).

\begin{proof}[Proof of Theorem \ref{tres-asp}~(i)]
Let $T$ be an $\ASPT$ proof for a tight normal logic program $\prog$.
Without loss of generality, we will assume that branches in $T$ have
not been extended further after they have become contradictory.
We now show that we can construct a $\TRES$ proof~$\pi$ for
$\comp(\prog)$ using the cut tree $T'$ corresponding to $T$.  
Furthermore, we show that for such a proof $\pi$ it holds that, given
any prefix $p$ of an arbitrary branch~$B$ in $T'$ there is a clause $C
\in \pi$ contradictory to the partial assignment in $p$, that is,
there is the entry~$\FALSE a$ ($\TRUE a$) for
$a\in\atoms(\prog)\union\body(\prog)$ in $p$ for each corresponding
positive literal $x_a$ (negative literal~$\bar x_a$) in $C$. 

Consider first the partial assignment in an arbitrary (full) branch
$B$ in $T'$.  
Assume that there is no clause in $\comp(\prog)$ contradictory to the
partial assignment in $B$, that is, we can obtain a truth assignment
$\tau$ based on the entries in $B$ such that every clause in
$\comp(\prog)$ is satisfied in $\tau$. But this leads to contradiction
since~$\comp(\prog)$ is satisfied if and only if $\prog$ is satisfied. 
Thus there is a clause $C \in \comp(\prog)$ contradictory to the
partial assignment in $B$, and we take the clause $C$ into our
resolution proof $\pi$.

Assume that we have constructed $\pi$ such that 
for any prefix $p$ of length $n$ for any branch $B$ in $T'$, 
there is a clause $C \in \pi$ contradictory to the partial assignment
in~$p$. 
Consider an arbitrary prefix $p$ of length $n-1$. 
Now, in $T'$ we have the  prefixes $p'$ and $p''$ of length $n$ which
have been obtained through extending $p$ by applying the cut rule on some
$a\in\atoms(\prog)\union\body(\prog)$.
In other words, $p'$ is $p$ with $\TRUE a$ appended in the end
($p''$ is $p$ with $\FALSE a$ appended in the end).
Since $p'$ ($p''$, respectively) is of length $n$, there is a clause
$C$ ($D$, respectively) in $\pi$ contradictory to the 
partial assignment in $p'$ ($p''$, respectively). 
Now there are two possibilities. If $C=\{\bar x_a\}\union C'$ and
$D=\{x_a\}\union D'$, we can resolve on $x_a$ adding $C'\union D'$ to
$\pi$.
Thus we have a clause~$C'\union D'\in \pi$ contradictory to the
partial assignment in the prefix $p$.
Otherwise we have that $\bar x_a\not\in C$ or $x_a\not\in D$, and  hence
either $C\in \pi$ or $D\in \pi$ is contradictory to the partial
assignment in the prefix $p$.

When reaching the root of~$T'$, we must have derived $\emptyset$ since
it is the only clause contradictory with the empty assignment.  
Furthermore, the $\TRES$ derivation $\pi$ is of polynomial length with
respect to $T'$ (and $T$). 
\end{proof}

The following example illustrates the $\RES$ proof construction used
above in the proof of Theorem \ref{tres-asp}~(i).
\begin{example}
Consider again the tight NLP 
$\Pi = \{a\IF b,\naf a.\; b\IF c.\; c\IF\naf b\}$ from Example
\ref{ex:completion} and the $\ASPT$ proof $T$ for $\prog$ in
Figure~\ref{tab-ex}. 
We now construct a $\TRES$ proof for the completion $\comp(\Pi)$ (see
Example \ref{ex:completion} for details) using the strategy from the
proof of Theorem \ref{tres-asp} (i). 
First, $T$ is transformed into a cut tree~$T'$
given in Figure~\ref{tres-ex} (left).
Consider now the two leftmost branches in $T'$. 
The partial assignment in the branch with entries~$\TRUE a$ and
$\FALSE\{b,\naf a\}$ is contradictory to clause $\{\bar x_a, x_{\{b,\naf
  a\}}\}$ in~$\comp(\prog)$, and the partial assignment in the
branch with entries~$\TRUE a$ and $\TRUE\{b,\naf a\}$ is
contradictory to clause $\{\bar x_{\{b,\naf a\}},\bar x_a\}$ in
$\comp(\prog)$.
Thus we resolve on $x_{\{b,\naf a\}}$ and obtain the clause $\{\bar
x_a\}$, which is contradictory to the single entry~$\TRUE a$ in the
prefix of the two leftmost branches in~$T'$. 
Similarly, we can construct a resolution tree for clause 
$\{x_a\}$ corresponding to the right side of $T'$. 
We finish the proof by resolving on~$x_a$.   
The complete $\TRES$ proof corresponding to the cut tree $T'$ is shown 
in Figure~\ref{tres-ex} (right). 
\end{example}
\begin{figure}
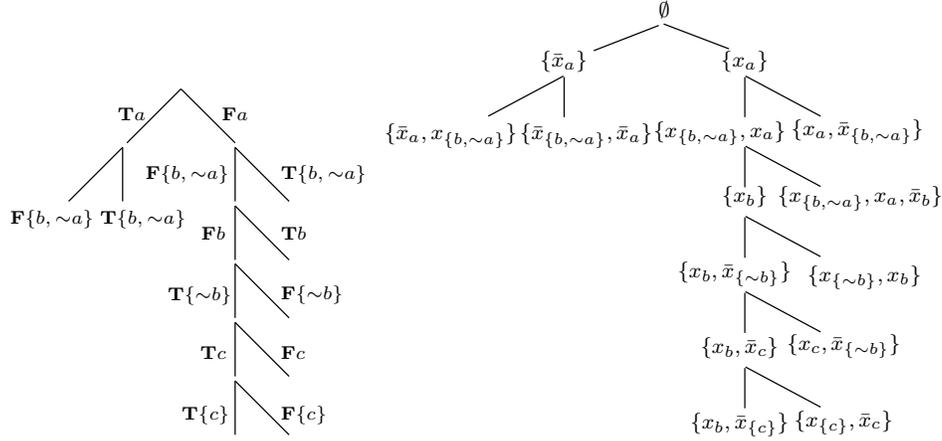

\begin{center}
\begin{tabular}[!hb]{cc}
\begin{minipage}[t]{0.43\textwidth}
\input{asp-tableu-cuts.pstex_t}
\end{minipage}
&
\begin{minipage}[t]{0.5\textwidth}
\input{res-proof.pstex_t}
\end{minipage}
\end{tabular}
\end{center}
\caption{Left: cut tree based on  the $\ASPT$ proof in
  Figure~\ref{tab-ex}. Right:  resulting $\TRES$
  proof\label{tres-ex}.}  
\end{figure}

\begin{proof}[Proof of Theorem \ref{tres-asp} (ii)]
Let $\pi = (C_1,\ldots, C_n=\emptyset)$ be a $\TRES$ refutation of a
set $\clauset$ of clauses.
Recall that each derived clause $C_i$ in $\pi$ is obtained by
resolving on $x$ from $C_{j}=C \cup \{x\}$ and $C_{k}= D \cup \{\bar
x\}$ for some $j,k<i$. 

An $\ASPT$ proof $T$ for $\nlp(\clauset)$ is obtained from $\pi$ as
follows. 
We start from $C_n$, which is obtained from clauses $C_j=\{x\}$ and
$C_k=\{\bar x\}$ by resolving on $x\in\vars(\clauset)$, and apply in
$T$ the cut rule on $a_x$ corresponding to $x$.
Then we recursively continue the same way with $C_j$ ($C_k$,
respectively) in the generated branch with $\FALSE a_x$ ($\TRUE a_x$,
respectively).  
Since $\pi$ is tree-like, each clause in the prefix $(C_1, \ldots,
C_{\max\{j,k\}})$ of $\pi$ is either used in the derivation of $C_j$
or $C_k$, but not in both.
By construction when reaching $C_1$ the branches of $T$ correspond
one-to-one to the paths in $\pi$ (seen as a tree) from $C_n$ to the
leaf clauses of $\pi$.  
For a particular leaf clause $C$, we have for each literal $l \in C$
($l=x$ or $l=\bar x$) contradicting entries for $a_x$ in the
corresponding branch of $T$, that is,
$\FALSE a_x$ if $l=x$ and
$\TRUE a_x$ if $l=\bar x$. 
Now we can directly deduce for each $\FALSE a_x$ 
the entry $\FALSE \{a_x\}$ and for each $\TRUE a_x$ the entry $\FALSE
\{\naf a_x\}$.  
These entries together will allow us to directly deduce $\FALSE c$
(all the bodies of rules with atom $c$ as the head are false).
Since we have $\bot \IF \naf c \in \nlp(\clauset)$, we can deduce
$\TRUE c$, and the branch becomes contradictory.
\end{proof}
The following example illustrates the strategy used in the proof of 
 Theorem~\ref{tres-asp}~(ii).

\begin{example}
\label{ex:tres-to-aspt}
Recall the set  $\clauset=\{ \{x, y\},\{
x,\bar y\}, \{\bar x, y\}, \{\bar x, \bar y\} \}$  of clauses
and the corresponding normal logic program $\nlp(\clauset)$  
presented in Example \ref{ex:clauses2nlp}.
The set $\clauset$ of clauses  has a $\TRES$
refutation $\pi=(\{x, y\},\{x,\bar y\}, \{\bar x, 
y\}, \{\bar x, \bar y\}, \{y\}, \{\bar y\}, \emptyset).$
Now we construct an $\ASPT$ proof $T$ for $\nlp(\clauset)$
from $\pi$ as done in the proof of Theorem~\ref{tres-asp}~(ii). 
The resulting $\ASPT$ proof $T$ is presented in
Figure~\ref{ex:from-res2aspt}.    
In the tableau we have omitted entries of the form $\TRUE\{l\}$ and
$\FALSE\{l\}$ for bodies consisting of a single default literal.
The empty clause is obtained resolving on $y$ from $\{y\}$
and $\{\bar y\}$, and thus we start with applying the cut rule on
$a_y$. 
The clause $\{\bar y\}$ is obtained resolving on $x$ from $\{x, \bar
y\}$ and $\{\bar x, \bar y\}$. 
We continue in the branch with $\TRUE a_y$ by applying the cut rule on~$a_x$. 
Consider now the branch with~$\TRUE a_y$ and~$\TRUE a_x$ in the tableau. 
The branch corresponds to the clause $\{\bar x, \bar y\}$ in
$\clauset$. Thus we arrive in a contradiction by deducing 
$\FALSE c_4$ from $c_4\IF \naf a_x$ and $c_4\IF \naf a_y$, and 
$\TRUE c_4$ from $\bot \IF \naf c_4$. 
Other branches become contradictory similarly.
\end{example}
\begin{figure}[!h]
\input{aspt-proof.pstex_t}
\caption{
An $\ASPT$ proof for $\nlp(\clauset)$ resulting from a $\TRES$ proof
$\pi=(\{x, y\},\{x,\bar y\}, \{\bar x, y\}, \{\bar x, \bar y\}, \{y\},
\{\bar y\}, \emptyset)$ for $\clauset$ in Example
\ref{ex:tres-to-aspt}\label{ex:from-res2aspt}.} 
\end{figure}


\section{Extended ASP Tableaux}
\label{easp}
We will now introduce an \emph{extension rule}\footnote{Notice that
  the extension rule introduced here differs from the one proposed
  in~\cite{HaiJY03} in the context of theorem proving.} to $\ASPT$,  
which results in \emph{Extended ASP Tableaux} ($\EASPT$), an extended
tableau proof system for ASP.
The idea is that one can define names for conjunctions of default
literals.
\begin{definition}

Given a normal logic program $\prog$ and two literals
$l_1,l_2 \in \dlits(\prog)$,
the (elementary) \emph{extension rule} in $\EASPT$
adds the rule $p \IF l_1,l_2$ to
$\prog$, where $p \not\in \atoms(\prog) \cup \{\bot\}$.
\end{definition}
It is essential that $p$ is a new atom for preserving satisfiability.
After an application of the extension rule one considers the program
$\prog'=\prog\union \{p \IF l_1,l_2\}$ instead of the original 
program $\prog$. 
Notice that $\atoms(\prog')=\atoms(\prog)\union\{p\}$. 
Thus when the extension rule is applied several times, the atoms
introduced in previous applications of the rule can be used in
defining further new atoms (see the forthcoming 
Example \ref{ex:extension-rule}). 

When convenient,  we will apply a generalization of the elementary
extension rule. 
By allowing one to introduce multiple bodies for $p$, the \emph{general
extension rule} adds a set of rules
$$\bigcup_{i} \{p \IF 
l_{i,1},\ldots,l_{i,k_i} \mid   p \not\in
\atoms(\prog)\union\{\bot\} \mbox{ and } l_{i,k} \in \dlits(\prog)
\mbox{ for all } 1\leq k \leq k_i\}$$
into $\prog$.
Notice that equivalent constructs can be introduced with the
elementary extension rule. 
For example, bodies with more than two literals can be decomposed with
balanced  parentheses using additional new atoms. 

\begin{example}
\label{ex:extension-rule}
Consider a normal logic program $\prog$ such that
$\atoms(\prog)=\{a,b\}$. We apply the general extension rule and add a
definition for the disjunction of atoms~$a$ and $b$, resulting in a
program $\prog\union\{c\IF a.\; c\IF b\}.$
An equivalent construct can be introduced by applying the elementary
extension rule twice: add first a rule $d\IF \naf a, \naf b$, and
then add a rule $c\IF \naf d, \naf d$. 
\end{example}

An $\EASPT$ proof for (the unsatisfiability of) a
program $\prog$ is an $\ASPT$ proof $T$ for $\prog \cup E$, where~$E$
is a set of \emph{extending (program)  rules}  generated with the
extension rule in $\EASPT$. 
The length of an $\EASPT$ proof is the length of $T$ plus the
number of program rules in~$E$.

A key point is that applications of the extension rule do 
not affect the existence of stable models.
\begin{theorem}
Extended ASP Tableaux is a sound and complete proof system for normal
logic programs.
\end{theorem}
\begin{proof}
Let $T$ be an $\EASPT$ proof for normal logic program $\prog$ with
the set~$E$ of extending rules, that is, an $\ASPT$ proof for $\prog \cup E$.
Since $\ASPT$ is sound and complete, there is a complete
non-contradictory branch in $T$ if and only if $\prog \cup E$
is satisfiable.
The set $\atoms(\prog)$ is a splitting set for $\prog\cup E$, since
$\head(r)\not\in\atoms(\prog)\union\{\bot\}$ for 
every extending rule $r \in E$.
Furthermore, \mbox{$\bottom{\prog\cup E}{\atoms(\prog)}=\prog$} and
\mbox{$\topp{\prog\cup E}{\atoms(\prog)}=E$}.  
By Theorem \ref{splitting}, $\prog \cup E$ is satisfiable if and only
if there is a solution to $\prog\union E$ with respect to $\atoms(\prog)$,
that is, there is a stable model $M\subseteq\atoms(\prog)$ for $\prog$
and a stable model $N$ for $\mathsf{eval}(E,M)$.
Since the rules in $E$ are generated using the extension rule (recall
also $\bot\not\in\head(E)$), there is a unique stable model for 
$\mathsf{eval}(E,M)$ for each $M\subseteq\atoms(\prog)$. 
Thus there is a solution to~$\prog\union E$ with respect to
$\atoms(\prog)$ if and only if $\prog$ is satisfiable, and moreover, 
$\prog\cup E$ is satisfiable if and only if $\prog$ is
satisfiable, and $\EASPT$ is sound and complete. 
\end{proof}

\subsection{The Extension Rule and Well-Founded Deduction}

An interesting question regarding the possible gains of applying
the extension rule in $\EASPT$ with the ASP tableau rules is whether
the additional extension rule allows one to simulate
well-founded deduction (rules (h$\dag$),(h$\ddag$),(i$\dag$), and 
(i$\ddag$)) with the other deduction rules
((b)--(g),(h$\S$),(i$\S$))\footnote{Notice that the proof system
consisting of tableau rules (a)--(g),(h$\S$), and (i$\S$) 
amounts to computing supported models~\cite{DBLP:conf/iclp/GebserS06}.}.
We now show that this is not the case; the extension rule
does not allow us to simulate reasoning related to unfounded sets and
loops.
This is implied by Theorem \ref{must-have-loop-rules}, which states
that, by removing rules (h$\dag$),(h$\ddag$),(i$\dag$), and (i$\ddag$)
from $\EASPT$, the resulting tableau method becomes incomplete for NLPs.

\begin{theorem}
\label{must-have-loop-rules}
Using only tableau rules (a)--(g), (h$\S$) and (i$\S$), and the
extension rule
does not result in a complete proof system for normal logic programs.
\end{theorem}

\begin{proof}
Consider the NLP $\prog=\{\bot\IF\naf a\rsep a\IF b\rsep b\IF a\}$.
Although $\prog$ is  unsatisfiable, in the proof system having only
the tableau rules (a)--(g),(h$\S$), and (i$\S$), we can construct a
complete and \emph{non-contradictory} tableau with a single branch
$$T=(\prog\cup\{\FALSE \bot\}, \FALSE\{\naf a\}\mbox{ (e)}, \TRUE a\mbox{
  (c)}, \TRUE\{b\}\mbox{ (i}\S\mbox{)}, \TRUE b \mbox{ (g)},
\TRUE\{a\}\mbox{ (i}\S\mbox{)})$$ for $\prog$.

Consider an arbitrary set~$E$ of extending rules generated
using the extension rule in $\EASPT$. 
Recall that $\head(E)\isect(\atoms(\prog)\union \{\bot\})=\emptyset$.
We can form a complete non-contradictory tableau $T'$ for
$\prog\union E$ as follows. 

First, define $H_0=\atoms(\prog)\union\{\bot\}$ and 
$$H_i=\{h\in\head(E)\mid 
\bigcup_{r\in\rules(h)}(\body(r)^+\union\body(r)^-)\subseteq
\bigcup_{j<i}H_j\}.$$
Thus the sets $H_i$ are used to define a level numbering for the atoms
defined in the extension $E$.
Furthermore, we define
$$E_i=\{r\in \prog\union E\mid \head(r)\in \bigcup_{j\leq i} H_j\}$$ 
for all $i\geq 0$.
Notice that $E_0=\prog$, and $\prog\union E=\bigcup_{i\geq 0}E_i$.  
We now show using induction that for each $i\geq 0$, the only branch
$B$ in $T$ can be extended into a complete non-contradictory branch 
for $E_i$ using tableau rules (b)--(g), (h$\S$), and~(i$\S$). 

The base case ($i=0$) holds by definition.
Assume that the claim holds for~\mbox{$i-1$}, that is, $B$ can be
extended into a complete non-contradictory branch $B'$ 
for $E_{i-1}$.
Consider now arbitrary $r\in E_i$. 
By definition $\body(r)^+\union\body(r)^-\subseteq
\atoms(E_{i-1})$ for each $r\in E_i$.  
Since $B'$ is complete, it contains entries for each 
$a\in\atoms(E_{i-1})$, and we can deduce an entry for $\body(r)$ 
using ASP tableau rule (b) or (f) (depending on the entries in $B'$). 
If the entry $\TRUE(\body(r))$ has been deduced, we can deduce~$\TRUE
h$ for $h=\head(r)$ using (d).
Otherwise, we have deduced the entries $\FALSE(\body(r'))$ for every
$r'\in E_i$ such that $h=\head(r')$, and we can deduce $\FALSE h$
using (h$\S$).   
Thus we have deduced entries for all $a\in\atoms(E_i)\union\body(E_i)$
and the branch is non-contradictory. 
Furthermore it is easy to check that the branch is closed under the
tableau rules (b)--(g),(h$\S$), and (i$\S$).

Thus we obtain a complete and non-contradictory tableau for $\prog\union
E$. 
Since we cannot generate a contradictory tableau for $\prog$
with tableau rules (a)--(g),(h$\S$), and (i$\S$), we cannot generate one
for $\prog\union E$ either.
This is in contradiction with the fact that $\prog$ is unsatisfiable.
\end{proof}


\section{Proof Complexity}
\label{comp}
In this section we study proof complexity theoretic issues related to
$\EASPT$ from several viewpoints: we will
\begin{itemize}
\item
consider the relationship between $\EASPT$ and the Extended
Resolution proof system~\cite{Tseitin:complexity}, 
\item
give an explicit separation of $\EASPT$ from 
$\ASPT$, and 
\item
relate the extension rule to the effect of program simplification on
proof lengths in $\ASPT$. 
\end{itemize}

\subsection{Relationship  with Extended Resolution}
The system $\EASPT$ is motivated by Extended Resolution ($\ERES$), a
proof system originally introduced in~\cite{Tseitin:complexity}.
The system $\ERES$ consists of the resolution rule and an extension
rule that allows one to expand a set of clauses 
by iteratively introducing equivalences of the form $x \equiv
l_1 \wedge l_2$, where $x$ is a new variable, and~$l_1$ and $l_2$ are 
literals in the current set of clauses.
In other words, given a set $\clauset$ of clauses, one application of the
extension rule adds the clauses  $\{x, \bar l_1, \bar l_2\}$, $\{\bar
x , l_1\}$, and $\{\bar x, l_2\}$ to~$\clauset$.
The system $\ERES$ is known to be more powerful than $\RES$; in
fact, $\ERES$ is polynomially equivalent to, for example,
extended Frege systems, and no superpolynomial proof complexity 
lower bounds are known for $\ERES$.
We will now relate $\EASPT$ with $\ERES$, and 
show that they are polynomially equivalent under the translations $\comp$ and
$\nlp$.
\begin{theorem} 
\label{eres-easp}
$\ERES$ and $\EASPT$
are polynomially equivalent proof systems 
in the sense that 
\begin{itemize}
\item[(i)] considering tight normal logic programs, $\ERES$
under the translation $\comp$ polynomially simulates 
$\EASPT$, and 
\item[(ii)]
considering sets of clauses, 
$\EASPT$ under the translation $\nlp$ polynomially simulates
$\ERES$.
\end{itemize}
\end{theorem}
\begin{proof} 
(i): 
Let $T$ be an $\EASPT$ proof for a tight NLP
 $\prog$, that is, $T$ is an $\ASPT$ proof for $\prog\union
E$, where $E$ is the set of extending rules generated in the proof.
We use the shorthand $x_l$ for the variable corresponding to default
literal $l$ in $\comp(\prog\union E)$, that is, $x_{l}=x_a$
($x_{l}=\bar x_a$, respectively) if $l=a$ ($l=\naf a$, respectively) 
for $a\in\atoms(\prog\union E)$.  
By Theorem~\ref{tres-asp} there is a polynomial
$\RES$ proof for $\comp(\prog\union E)$.
Now consider $\comp(\prog)$.
We apply the extension rule in  $\ERES$ in
the same order in which the extension rule in $\EASPT$ is applied when
generating the set $E$ of extending rules.
In other words, we apply the  extension rule in $\ERES$ as follows 
for each rule $r = h\IF l_1, l_2$ in $E$.
If $\body(r) = \{l_1, l_2\}\in\body(\prog)$, then there are the clauses
$x_{\{l_1, l_2\}}\equiv x_{l_1}\wedge x_{l_2}$ in $\comp(\prog)$.
If this is the case, we generate
the clauses $x_h \equiv x_{\{l_1, l_2\}}$ with the extension rule in
$\ERES$.
Otherwise, that is, if $\body(r)$ does not have a corresponding
propositional variable in $\comp(\prog)$, 
we generate the clauses
$x_h \equiv x_{\{l_1, l_2\}}$ and $x_{\{l_1, l_2\}}\equiv x_{l_1}\wedge x_{l_2}$.
Denote the resulting set of extending clauses by $E'$.
Now we notice that 
$\comp(\prog)\union E'=\comp(\prog\union E)$, and therefore
the $\RES$ proof for $\comp(\prog\union E)$ is an
$\ERES$ proof for $\comp(\prog)$ in which the extension rule in $\ERES$
 is applied to generate the clauses in $E'$.

(ii): 
Let $\pi = (C_1,\ldots, C_n=\emptyset)$ be an $\ERES$
proof for a set $\clauset$ of clauses. Let $E$ be the set of clauses
in $\pi$ generated with the extension rule.
We introduce shorthands for atoms corresponding to
literals, that is, $a_{l}=a_x$ ($a_{l}=\naf a_x$)
if $l=x$ ($l=\bar x$) for $x\in\vars(\clauset\union E)$. 
Now, an $\EASPT$ proof for $\nlp(\clauset)$  is generated as follows.
First, we add the following rules to $\nlp(\clauset)$ with the
extension rule in $\EASPT$: 
\begin{eqnarray}
&& \hspace{-5ex}a_x\IF a_{l_1}, a_{l_2}\mbox{ for each extension
}x\equiv l_1\land l_2; \label{exten-rules}\\
&&\hspace{-5ex}c\IF a_{l}\mbox{ for each literal }l\in C\mbox{ for a
  clause }C\in \pi\mbox{ such that }C\not\in\clauset;\mbox{ and}
\label{clause-rules}\\ 
&&\hspace{-5ex} p_{1}\IF c_1\mbox{ and }p_{i}\IF c_i, p_{{i-1}}\mbox{
  for each }C_i\in\pi \mbox{ and }2\le i<n. \label{proof-rules}
\end{eqnarray}
Then, from $i=1$ to $n-1$ apply the cut rule on $p_{i}$ in the branch with 
$\TRUE p_{j}$ for all $j<i$.
We now show that for each $i$ the branch with $\FALSE p_i$ and $\TRUE p_j$
for all $j<i$ becomes contradictory without further application of the
cut rule. 
First, deduce~$\FALSE c_i$ from $\FALSE p_i$ using the rule 
(\ref{proof-rules}) for $i$.
One of the following holds for $C_i\in\pi$: either
(a)~$C_i\in\clauset$, (b) $C_i$ is a derived clause, or  (c) $C_i\in
E$.  
\begin{itemize}
\item[(a)]
If $C_i\in\clauset$ we can deduce~$\TRUE c_i$ from $\bot\IF \naf
c_i\in\nlp(\clauset)$, and the branch becomes contradictory. 
\item[(b)]
If  $C_i$ is a derived clause, that is, $C_i$ is obtained from $C_j$ and
$C_k$ for $j,k<i$ resolving on $x$, then $C_i=(C_k\union
C_j)\setminus\{x,\bar x\}$.  
For all the literals $l\in C_i$ we deduce~$\false a_l$ from the rules
(\ref{clause-rules}) in the extension. 
From $\TRUE p_j$ and $\TRUE p_k$ we deduce~$\TRUE c_j$ and~$\TRUE c_k$ 
using the rule (\ref{proof-rules}) in the extension for $j$ and $k$, 
respectively. 
Furthermore because we have entries $\false a_l$ for each $l$ in  
$(C_k\union C_j)\setminus\{x,\bar x\}$, we deduce
$\TRUE a_x$ and $\FALSE a_x$ and the branch becomes contradictory.
Recall that there is a rule $c\IF a_l$ for each clause $C\in \pi$ and
literal $l\in C$ either in $\nlp(\clauset)$ or in the extension (rules
in (\ref{clause-rules})).  
\item[(c)]
If $C_i\in E$, then $C_i$ is of the form  $\{x, \bar l_1, \bar l_2\}$,
$\{\bar x, l_1\}$, or $\{\bar x, l_2\}$
for $x\equiv l_1\land l_2$. 
For instance, if $C_i=\{\bar x, l_1\}$, then from
$c_i\IF \naf a_x$ and $c_i\IF a_{l_1}$ we deduce~$\TRUE a_x$ 
and~$\false a_{l_1}$. The branch becomes contradictory as $\TRUE \{ a_{l_1},
a_{l_2} \}$ and~$\true a_{l_1}$ are deduced from 
a rule~(\ref{exten-rules})
in the extension. 
The branch becomes contradictory similarly, if~$C_i$ is of the form
$\{x, \bar l_1, \bar l_2\}$ or $\{\bar x, l_2\}$.
\end{itemize}
Finally, consider the branch with $\TRUE p_i$ for all $i=1\ldots n-1$.
The empty clause $C_n$ in~$\pi$ is obtained by resolving $C_{j} = \{x\}$ and
$C_k = \{\bar x\}$ in~$\pi$ for some $j,k<n$.
Thus we can deduce $\TRUE c_{j}$ and $\TRUE c_k$ from 
rules (\ref{proof-rules}) for $j$ and $k$,
respectively, and furthermore,~$\TRUE a_x$ 
and~$\FALSE a_x$ from $c_{j}\IF a_x$
and $c_{k}\IF \naf  a_x$,
resulting in a contradiction in the branch.
The obtained contradictory ASP tableau is of linear length with
respect to $\pi$.
\end{proof}

\subsection{Pigeonhole Principle Separates Extended ASP Tableaux from
  ASP Tableaux} 
To exemplify the strength of $\EASPT$, we now consider  a family of
normal logic programs $\{\prog_n\}$ which separates $\EASPT$ from
$\ASPT$, that is, we give an explicit polynomial-length proof for 
$\prog_n$  for which $\ASPT$ has exponential-length minimal 
proofs with respect to $n$.
We will consider this family also in the experiments reported in this
article.

The program family $\{\PHP\}$ in question is the following typical 
 encoding of the \emph{pigeonhole principle} as a normal logic program:
\begin{eqnarray}     
\PHP &= &
\{\bot \IF \naf p_{i,1}, \ldots, \naf
p_{i,n} \mid 1 \le i \le n+1\}\; \cup 
 \label{php1} \\ 
&&
\{\bot \IF p_{i,k}, p_{j,k}\mid 1 \le i < j \le n+1,\
1 \le k \le n \}\; \cup \label{php2} \\ 
&&
\{p_{i,j} \IF \naf p_{i,j}'\rsep\; p_{i,j}' \IF \naf p_{i,j}\mid
1 \le i \le n+1,\ 1 \le j \le n\}.
\label{php3}
\end{eqnarray}       
In the program above, $p_{i,j}$ has the interpretation that pigeon $i$
sits in hole $j$.
The rules in~(\ref{php1}) require that each pigeon must sit in some hole,
and the rules in~(\ref{php2}) require that no two pigeons can sit in
the same hole. The rules in~(\ref{php3}) enforce that for each pigeon
and each hole, the pigeon either sits in the hole or does not sit in
the hole. 
Each $\PHP$ is unsatisfiable since there is no bijective mapping from
an $(n+1)$-element set to an $n$-element set. 

\begin{theorem}
The complexity of $\{\PHP\}$ with respect to $n$
is 
\begin{itemize}
\item[(i)] polynomial in $\EASPT$,  and
\item[(ii)] exponential in $\ASPT$.
\end{itemize}
\label{separate}
\end{theorem}

\begin{proof}
(i): 
In~\cite{Cook:short} an extending set of clauses is added
to a clausal encoding $\clauset_{\mathrm{PHP}}$ of the pigeonhole 
principle\footnote{The particular encoding, for which there are no
  polynomial-length $\RES$ proofs~\cite{Haken:resolution}, is 
$\clauset_{\mathrm{PHP}}=\bigcup_{1 \le i \le n+1} 
\{\{\bigvee_{j=1}^n p_{i,j}\}\} \cup 
\bigcup_{1 \le i < j \le n+1,1 \le k \le n}
\{\{\neg p_{i,k} \vee \neg p_{j,k}\}\}$.}
so that $\RES$ has polynomial-length proofs for the resulting set of
clauses. 
By Theorem~\ref{eres-easp} (ii) there is a polynomial-length $\EASPT$
proof for  
\begin{eqnarray*}
\nlp(\clauset_{\mathrm{PHP}}) &=&
\{p_{i,j} \IF \naf p_{i,j}'\rsep\; p_{i,j}' \IF \naf p_{i,j}\mid
1 \le i \le n+1,\ 1 \le j \le n\}\cup \\
&&
\{\bot \IF \naf c_i
 \mid 1\leq i\leq n+1\}\cup  \\
&& \{\bot \IF \naf c_{ijk} \mid 1\leq i<j\leq n+1,1\leq k\leq n\}\cup \\
&& 
\{c_i \IF p_{i,j} \mid 1\leq j\leq n, 1\leq i\leq n+1\} \cup \\
&& 
\{c_{ijk} \IF \naf p_{i,k}\rsep\;
c_{ijk} \IF \naf p_{j,k} \mid 1\leq i<j\leq n+1,1\leq k\leq n\}.
\end{eqnarray*}
For simplicity, we keep the names of the atoms $p_{i,j}$ unchanged in
the translation.

In more detail, let
$\pi = (C_1,C_2,\ldots, C_m =\emptyset)$ be the polynomial-length
$\ERES$ proof\footnote{
The polynomial-length $\ERES$ proof for $\clauset_{\mathrm{PHP}}$
is not described in detail in~\cite{Cook:short}.
Details on the structure of the $\RES$ proof can be found
in~\cite{JarvisaloJ:Constraints08}.
The intuitive idea is that the extension allows for reducing $\PHP$ to
$\mathrm{PHP}_{n-1}^n$ with a polynomial number of resolution steps.}
for the clausal representation $\clauset_{\mathrm{PHP}}$.
Let
\begin{eqnarray*}
\EXT &= &
\{e_{i,j}^l \IF e_{i,j}^{l+1}\!\rsep\;e_{i,j}^l \IF e_{i,l}^{l+1},
e_{l+1,j}^{l+1}
\mid 1 \le i \le l\mbox{ and }1 \le j \le l-1\}
\end{eqnarray*}
for $1< l \le n$, where each $e_{i,j}^{n+1}$ is $p_{i,j}$. The
extension $\EXT$ corresponds the set of extending clauses
in~\cite{Cook:short} similarly to the set of rules (\ref{exten-rules})
in part~(ii) of the proof of Theorem~\ref{eres-easp}.
Furthermore, $\E(\pi)$ consists of the sets of rules
(\ref{clause-rules}) and (\ref{proof-rules}) defined in 
the proof of Theorem~\ref{eres-easp} (ii).
By applying the strategy from the proof of
Theorem~\ref{eres-easp}~(ii), 
we obtain a polynomial-length~$\ASPT$ proof for  
\begin{equation*}
\nlp(\clauset_{\mathrm{PHP}})\cup \bigcup_{1 < l\le n}\EXT\union\E(\pi).
\end{equation*}

Now, we use the same strategy to construct a polynomial $\ASPT$ proof
for the program
$$\EPHP = \PHP \cup \bigcup_{1 < l\le n} \EXT \cup \E'(\pi),$$
where $\E'(\pi)$ consists of rules $c\IF a_{l}$ for each literal $l\in
C$ for each clause $C\in \pi$ (that is, rules as in (\ref{clause-rules})
but without the restriction $C\not\in\clauset_{\mathrm{PHP}}$)
together with the rules in~(\ref{proof-rules}). 
The only difference comes in step~(a) in the proof of
Theorem~\ref{eres-easp}~(ii), that is, when we have deduced $\FALSE c$ 
corresponding to $C\in\clauset_{\mathrm{PHP}}$. Since we do not have
the rule $\bot\IF \naf c$ in $\EPHP$, we cannot deduce $\TRUE c$
to obtain a contradiction.
Instead, we can deduce a contradiction without using the $\ASPT$ cut rule 
through a program rule in
$\PHP$ that corresponds to the clause $C$. 
For instance, if $C=\{\neg p_{i,k},\neg p_{j,k}\}$, we have
the rules $c\IF \naf p_{i,k}$ and $c\IF \naf p_{j,k}$ 
in $\E'(\pi)$ and the rule $\bot \IF p_{i,k}, p_{j,k}$ in $\PHP$. 
From~$\FALSE c$, we deduce~$\TRUE p_{i,k}$ and~$\TRUE p_{j,k}$. 
From~$\FALSE \bot$ and \mbox{$\bot \IF p_{i,k},p_{j,k}$}, we 
deduce~$\FALSE\{p_{i,k},p_{j,k}\}$, and furthermore, from $\TRUE p_{i,k}$
and $\FALSE\{p_{i,k},p_{j,k}\}$, we deduce~$\FALSE p_{j,k}$.
This results in a polynomial-length $\EASPT$ proof for $\PHP$.

(ii): 
Assume now that there is a polynomial $\ASPT$ proof for $\PHP$.
By Theorem~\ref{tres-asp}, there is a polynomial $\TRES$ proof for
$\comp(\PHP)$.  
Notice that the completion $\comp(\PHP)$ consists of the clausal
encoding $\clauset_{\mathrm{PHP}}$ of the pigeonhole principle 
and additional clauses (tautologies) for rules of the form
$p_{i,j}\IF \naf p_{i,j}'$, $p_{i,j}'\IF \naf p_{i,j}$. 
It is easy to see that these additional tautologies 
do not affect the length of the minimal $\TRES$
proofs for $\comp(\PHP)$.
Thus there is a polynomial-length $\TRES$ proof for
the clausal pigeonhole encoding. 
However, this contradicts
the fact that the complexity of the clausal  pigeonhole principle is
exponential with respect to
 $n$ for (Tree-like) Resolution~\cite{Haken:resolution}.
\end{proof}

We can also easily obtain a \emph{non-tight}
program family to witness the separation demonstrated in 
Theorem~\ref{separate}. Consider the family 
$$\{\PHP \cup \{p_{i,j} \IF p_{i,j} \mid 1\le i \le n+1,1\le j \le n \}\},$$
which is non-tight with the additional self-loops $\{p_{i,j} \IF
p_{i,j}\}$, but preserves (un)satis\-fia\-bility of ${\rm PHP}^{n+1}_n$ for
all $n$.
Since the self-loops do not contribute to the proofs for $\PHP$,
$\ASPT$
 still has exponential-length minimal proofs for these
programs, while the polynomial-length $\EASPT$ proof presented in the proof of 
Theorem~\ref{separate} is still valid.

The generality of the arguments used in the proof of Theorem~\ref{separate} 
is not limited to the specific family $\PHP$ of NLPs.
For understanding the general idea behind the explicit 
construction of $\EPHP$, it is informative to notice the following.
Instead of considering $\PHP$, one can apply the argument 
in the proof Theorem~\ref{separate} using any tight NLP $\prog$ which
represents a set of clauses $\clauset$ for which
(i)~there is no polynomial-length $\RES$ proof, but for which
(ii)~there is a polynomial-length $\ERES$ proof .
By property (ii)~we know from Theorem~\ref{eres-easp}~(ii) that there
is a polynomial-length $\EASPT$ proof for $\prog$.

\subsection{Program Simplification and Complexity}
We will now give an interesting corollary of
Theorem~\ref{separate}, addressing the effect of program
simplification on the length of proofs in $\ASPT$.

Tightly related to the development of efficient solver
implementations for ASP programs arising from
practical applications is the development of techniques for
\emph{simplifying} programs.  
Practically relevant programs are often generated automatically, and
in the process a large number of redundant constraints is produced.
Therefore efficient program simplification through \emph{local
  transformation rules} is important.
While various satisfiability-preserving local transformation rules for 
simplifying logic programs have been introduced (see 
\cite{DBLP:conf/lpnmr/EiterFTW04} for example), the effect of applying such
transformations on the lengths of proofs has not received attention.

Taking a first step into this direction, we now show that even simple
transformation rules may have a drastic negative effect on proof
complexity. 
Consider the local transformation rule 
\begin{eqnarray*}
\red(\prog) &=& \prog \setminus \{r \in \prog \mid 
\head(r) \not\in\!\!\!\bigcup_{B\in\body(\prog)}\!\!(B^+\union B^-)
\mbox{ and } \head(r)\ne\bot\}.
\end{eqnarray*}
A polynomial-time simplification algorithm $\red^*(\prog)$ is obtained
by closing program~$\prog$ under $\red$.  
Notice that we have $\red^*(\EPHP) = \PHP$.
Thus, by Theorem~\ref{separate}, $\red^*$ transforms a
program family having  polynomial complexity in ASP Tableaux into one 
with exponential complexity with respect to $n$.

The rules removed by $\red^*$ are redundant with respect to
satisfiability of the program in the sense that $\red^*$ 
preserves \emph{visible
  equivalence}~\cite{DBLP:journals/jancl/Janhunen06}. 
The visible equivalence relation takes the interfaces of programs into
account: $\atoms(\prog)$ is partitioned into $\visible(\prog)$ and
$\hidden(\prog)$ determining the {\em visible} and the {\em hidden}
atoms in $\prog$, respectively. Programs $\prog_1$ and
$\prog_2$ are visibly equivalent, denoted by $\prog_1\lpeq{v}\prog_2$,
if and only if $\visible(\prog_1) = \visible(\prog_2)$ and there is a
bijective correspondence between the stable models of $\prog_1$ and
$\prog_2$ mapping each $a\in\visible(\prog_1)$ onto itself.
Now if one defines $\visible(\prog)=\atoms(\red^*(\prog))=
\visible(\red^*(\prog))$, that is, assuming that the atoms removed by
$\red^*$ are hidden in $\prog$, one can see that
$\red^*(\prog)\lpeq{v}\prog$.
Hence, even though there is a bijective correspondence between the
stable models of $\EPHP$ and $\red^*(\EPHP) = \PHP$, $\red^*$ causes a
superpolynomial blow-up in the length of proofs in $\ASPT$ and the
related solvers, if applied before actually proving $\EPHP$.


\section{Experiments}
\label{experiments}
We experimentally evaluate how well current state-of-the-art ASP
solvers can make use of the additional structure 
introduced to programs  using the extension rule.
For the experiments, we ran 
the solvers~\footnote{We note that the detailed results
  reported here differ somewhat from those reported in the conference
  version of this work~\cite{JarvisaloO:ICLP07}. This is 
due to the fact that, for the current
  article, we used more recent versions of the solvers.}
$\smodels$~\cite{DBLP:journals/ai/SimonsNS02} (version~2.33, a
widely used lookahead solver), 
$\clasp$~\cite{clasp} (version~1.1.0,  with many
techniques---including conflict learning---adopted from DPLL-based SAT
solvers), 
and 
$\cmodels$~\cite{GiunchigliaLM:answer} (version~3.77, a SAT-based ASP
solver running the  conflict-learning SAT solver 
zChaff~\cite{DBLP:conf/dac/MoskewiczMZZM01} version~2007.3.12 as the
back-end). 
The experiments were run on standard PCs with 2-GHz AMD 3200+
processors under Linux.
Running times were measured using \texttt{/usr/bin/time}.

First, we investigate whether ASP solvers are able to
benefit from the extension in $\EPHP$.
We compare the number of decisions and running times of each of the
solvers on $\PHP$, $\CPHP = \PHP \cup \bigcup_{1< l \le n} \EXT$, and
$\EPHP$. 
By Theorem~\ref{separate} the solvers should in theory be able to 
exhibit polynomially scaling numbers of decisions for $\EPHP$.
In fact with conflict-learning this might also be possible for $\CPHP$
due to the tight correspondence with conflict-learning SAT solvers and
$\RES$~\cite{Beame:understanding}.
The results for $n = 10\ldots12$ are shown in Table~\ref{tab}.
\begin{table}[b]
\caption{Results on $\PHP$, $\CPHP$, and $\EPHP$ 
with timeout (-) of 2 hours.}
\label{tab}
\begin{minipage}{\textwidth}
\begin{tabular}{cccccccc}
\hline
\hline
\multicolumn{2}{c}{} & \multicolumn{3}{c}{\textbf{Time (s)}}&
\multicolumn{3}{c}{\textbf{Decisions}}\\ \cline{3-5} \cline{6-8}
\\
\textbf{Solver} & $n$ &  $\PHP$ & $\CPHP$ & $\EPHP$
& $\PHP$ & $\CPHP$ & $\EPHP$\\
\hline
\hline
$\smodels$ & $10$ &34.02& 119.69& 8.65& 164382 &144416 &0 \\
$\smodels$ & $11$ &486.44& 1833.48& 21.70&1899598 & 1584488&0 \\
$\smodels$ & $12$ &-& - & 49.28& - & -&  0\\
\hline
$\clasp$ & $10$ & 6.81 & 7.29& 10.05 & 337818 & 216894&38863 \\
$\clasp$ & $11$ &58.48& 45.00& 82.07& 1840605 & 882393& 203466\\
$\clasp$ & $12$ &579.28 & 509.43& 941.23& 12338982 & 6434939& 1467623 \\
\hline
$\cmodels$ & $10$ &1.60 &1.69 &7.87 & 8755 & 8579& 12706\\
$\cmodels$ & $11$ &8.20& 8.51& 43.96 & 24318& 23758& 42782\\
$\cmodels$ & $12$ &46.33&54.26 & 122.72& 88419& 94917& 88499\\
\hline
\hline
\end{tabular}
\vspace{-2\baselineskip} 
\end{minipage}
\end{table}
While the number of decisions for the conflict-learning solvers
$\clasp$ and $\cmodels$ is somewhat reduced by the extensions, the
solvers do not seem to be able to reproduce the polynomial-length
proofs, and we do not observe a dramatic change in the running times.
With a timeout of 2 hours, $\smodels$ gives no answer for $n=12$ on 
$\PHP$ or $\CPHP$.
However, for $\EPHP$ $\smodels$ returns without any branching, which
is due to the fact that $\smodels$' complete lookahead notices
that by branching on the critical extension atoms (as in part (ii) of the
proof of Theorem~\ref{separate}) the $\F$ branch 
becomes contradictory immediately.
With this in mind, an interesting further study out of the scope of
this work would be the possibilities of integrating conflict learning
techniques with (partial) lookahead. 

In the second experiment, we study the effect of having a
modest number of redundant rules on the behavior of ASP
solvers.
For this we apply the procedure
$\textsc{AddRandomRedundancy}(\prog,n,p)$ shown in
Algorithm~\ref{random-redundancy}.
Given a program $\prog$, the procedure iteratively 
adds rules of the
form $r_i \IF l_1,l_2$ to $\prog$, where $l_1,l_2$ are random default
literals currently in the program and $r_i$ is a new atom.
The number of introduced rules is $p$\%  of the integer $n$.

\begin{algorithm}[!ht]
{\caption{$\textsc{AddRandomRedundancy}(\prog,n,p)$}\label{random-redundancy}}
\begin{enumerate}
\item[1.] \textbf{For} $i = 1$ \textbf{to} $\lfloor \frac{p}{100}n\rfloor$:
\begin{enumerate}
\item[1a.] Randomly select $l_1,l_2 \in \dlits(\prog)$ such that 
$l_1 \not =  l_2$.
\item[1b.] 
$\prog := \prog \cup \{r_i \IF l_1,l_2\}$, where $r_i \not
  \in \atoms(\prog)\union\{\bot\}$.
\end{enumerate}
\item[2.] \textbf{Return} $\prog$
\end{enumerate}
\end{algorithm}

\begin{figure}[!b]
\begin{center}
\begin{tabular}{c}
\begin{minipage}{12cm}
\begin{tabular}{cc}
\epsfxsize=6cm
\epsfbox{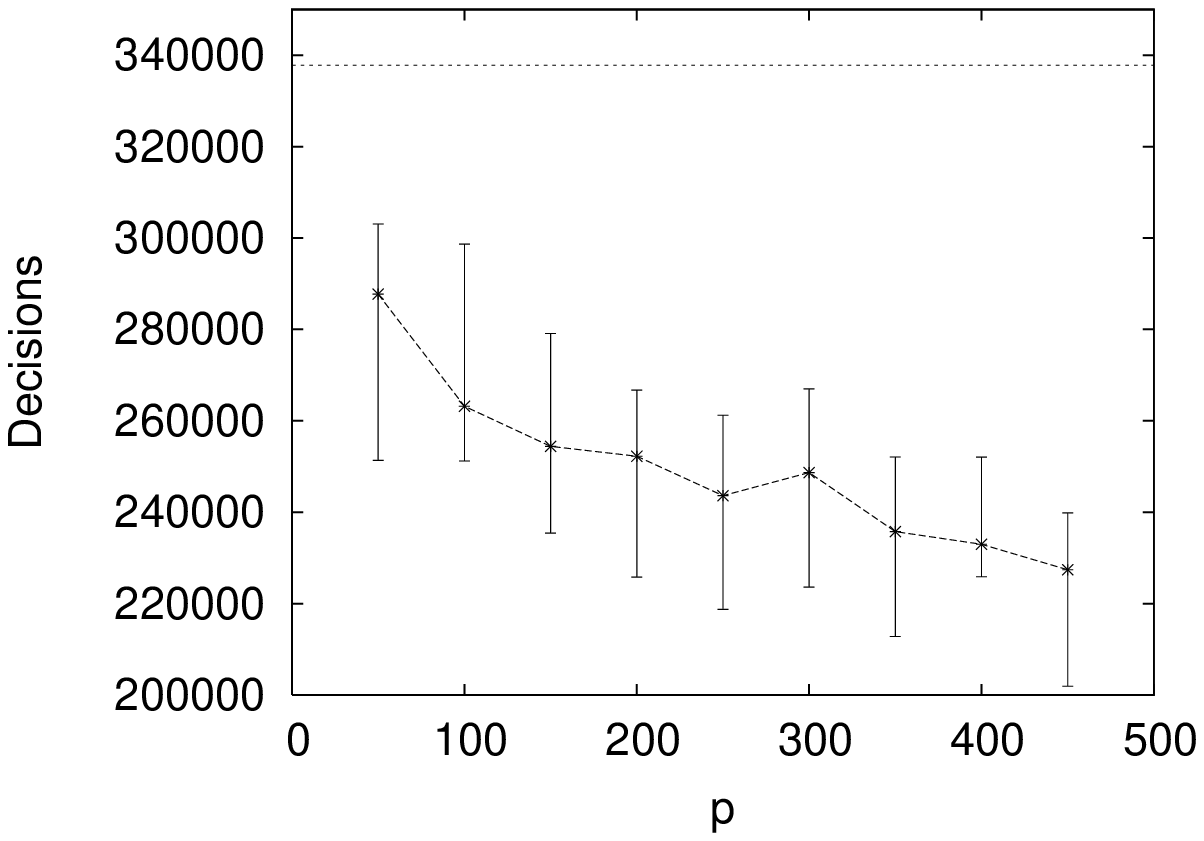}
& 
\epsfxsize=6cm
\epsfbox{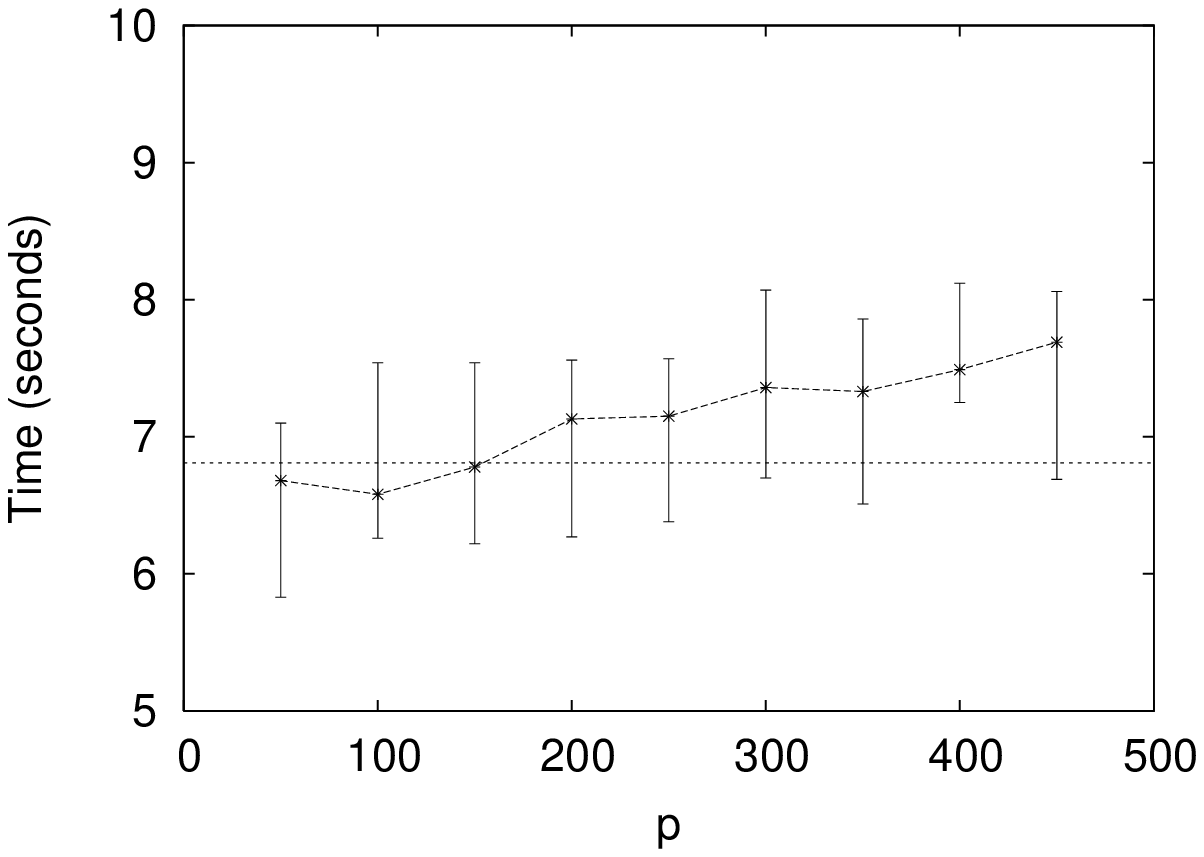}
\end{tabular}
\end{minipage}
\\
(a) $\clasp$ decisions (left), time in seconds (right)
\\
\begin{minipage}{12cm}
\begin{tabular}{cc}
\epsfxsize=6cm
\epsfbox{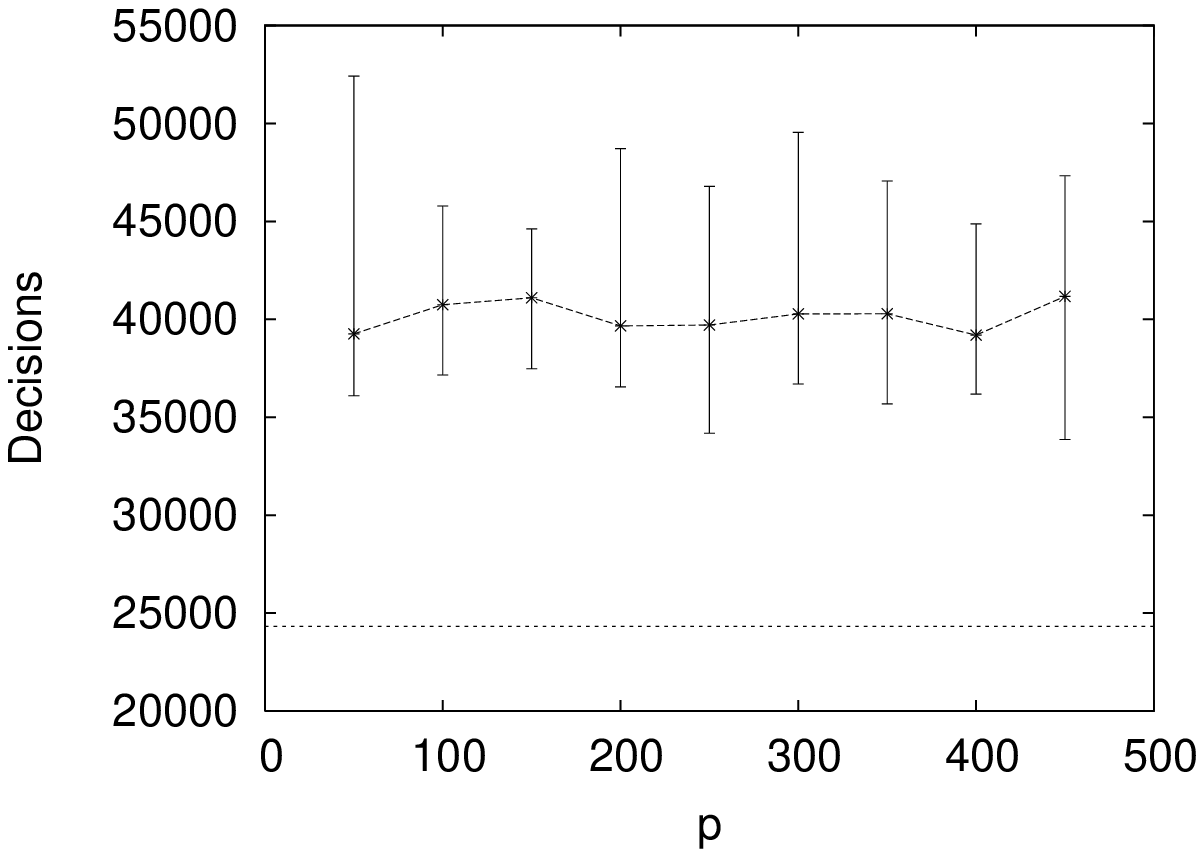}
& 
\epsfxsize=6cm
\epsfbox{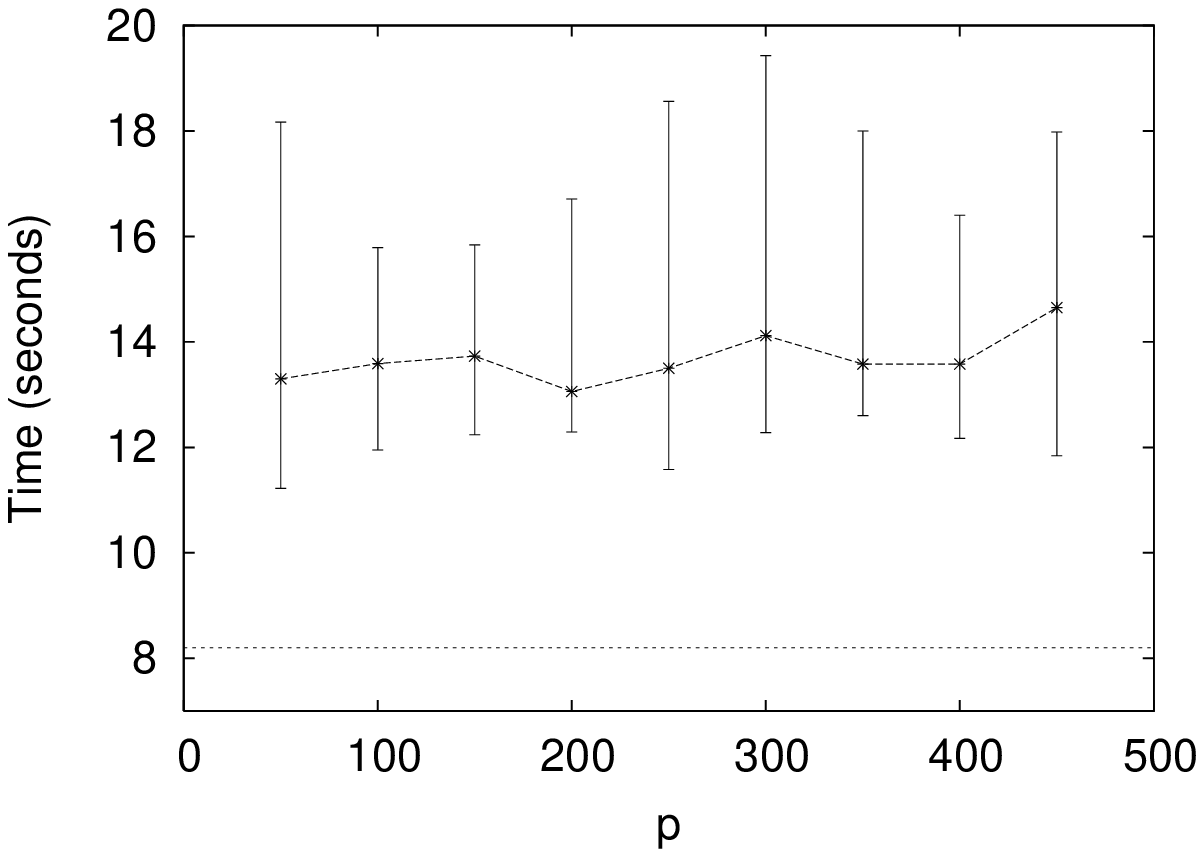}
\end{tabular}
\end{minipage}
\\
(b) $\cmodels$ decisions (left), time in seconds (right)
\\
\begin{minipage}{12cm}
\begin{tabular}{cc}
\epsfxsize=6cm
\epsfbox{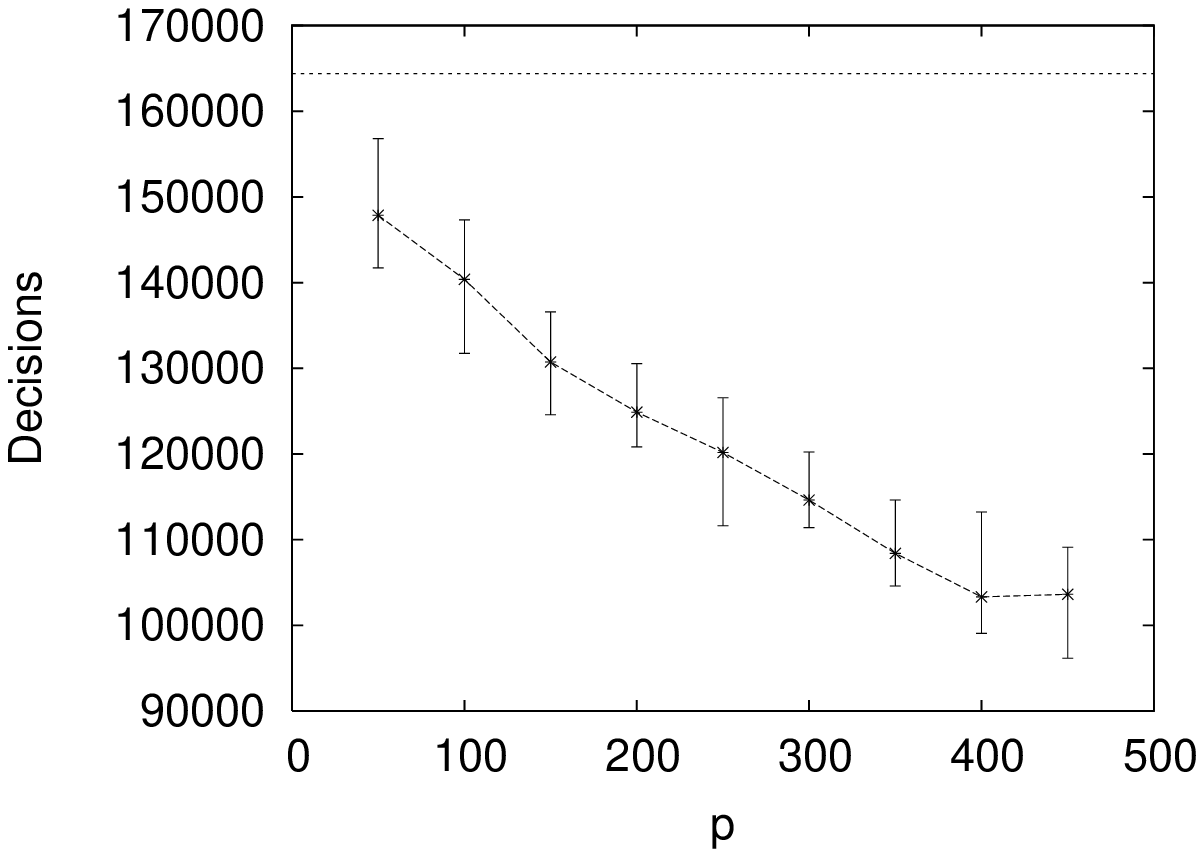}
& 
\epsfxsize=6cm
\epsfbox{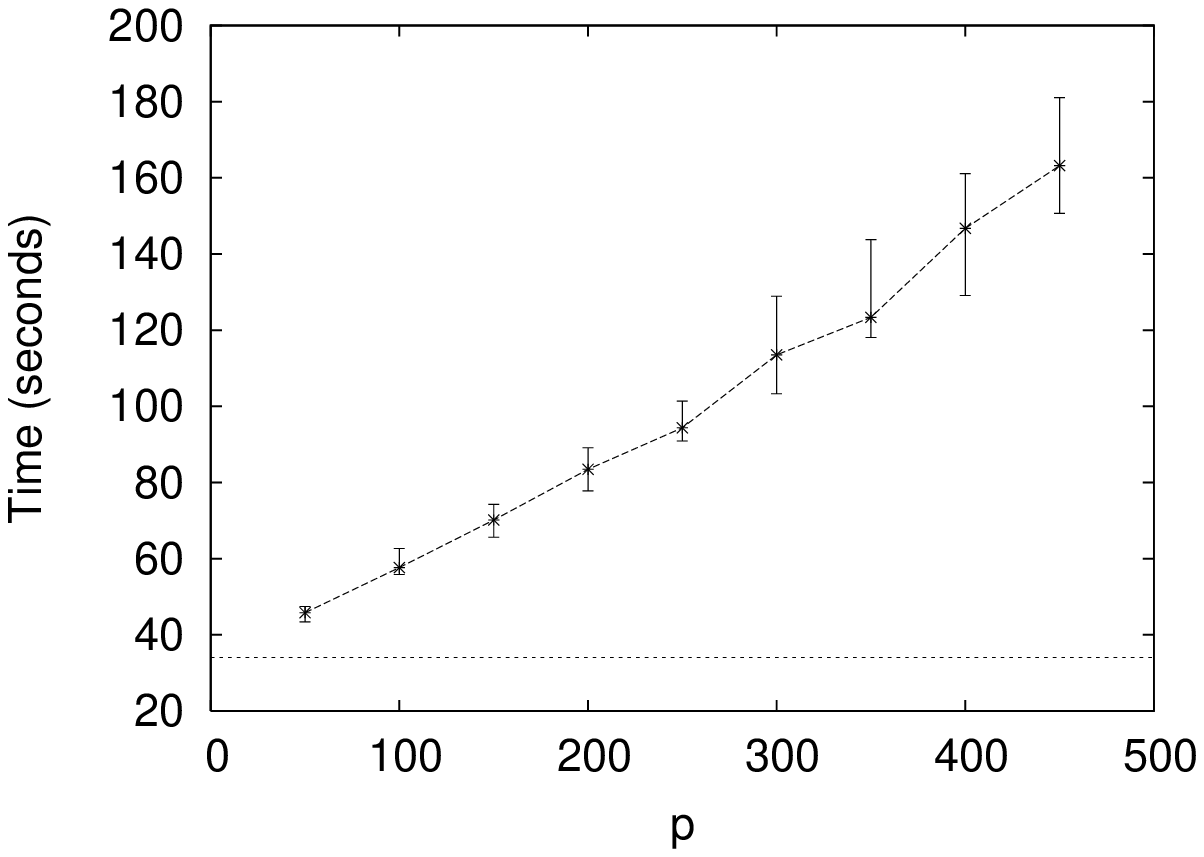}
\end{tabular}
\end{minipage}
\\
(c) $\smodels$ decisions (left), time in seconds (right)
\\
\begin{minipage}{12cm}
\begin{tabular}{cc}
\epsfxsize=6cm
\epsfbox{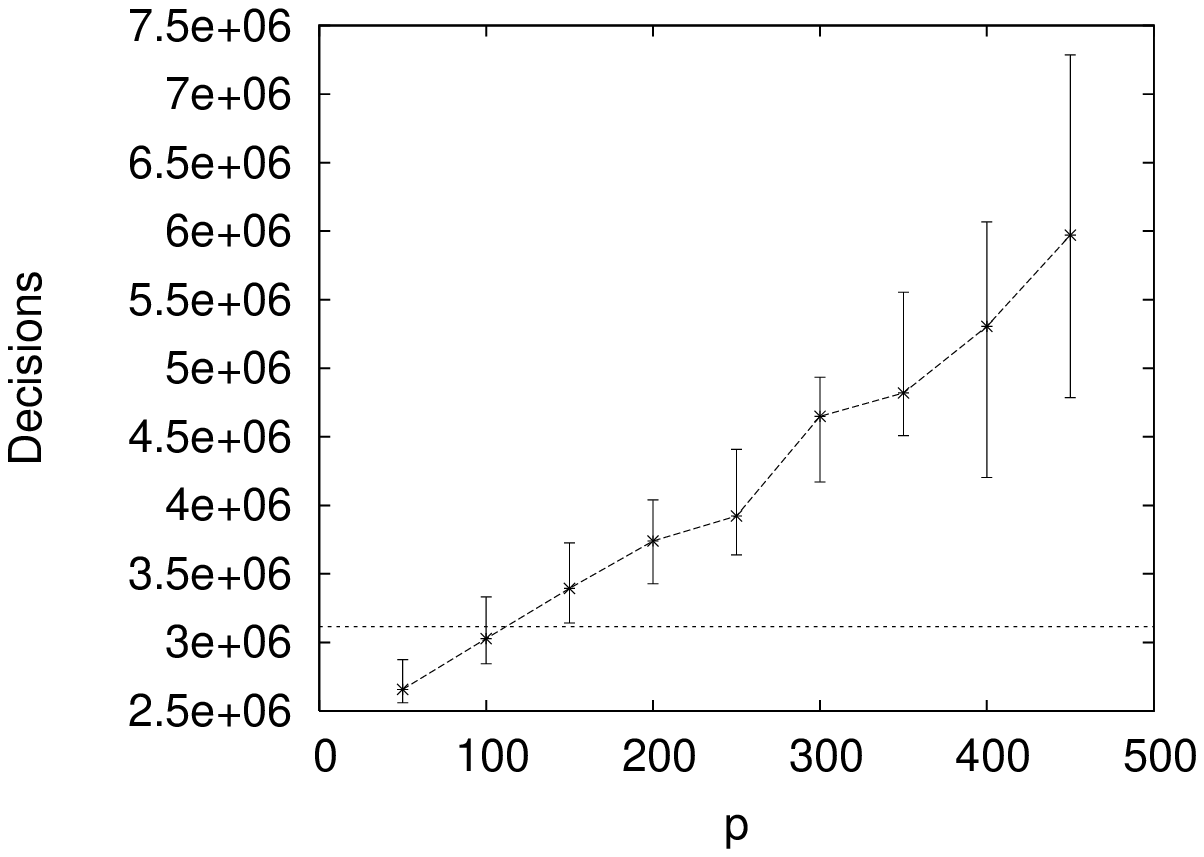}
& 
\epsfxsize=6cm
\epsfbox{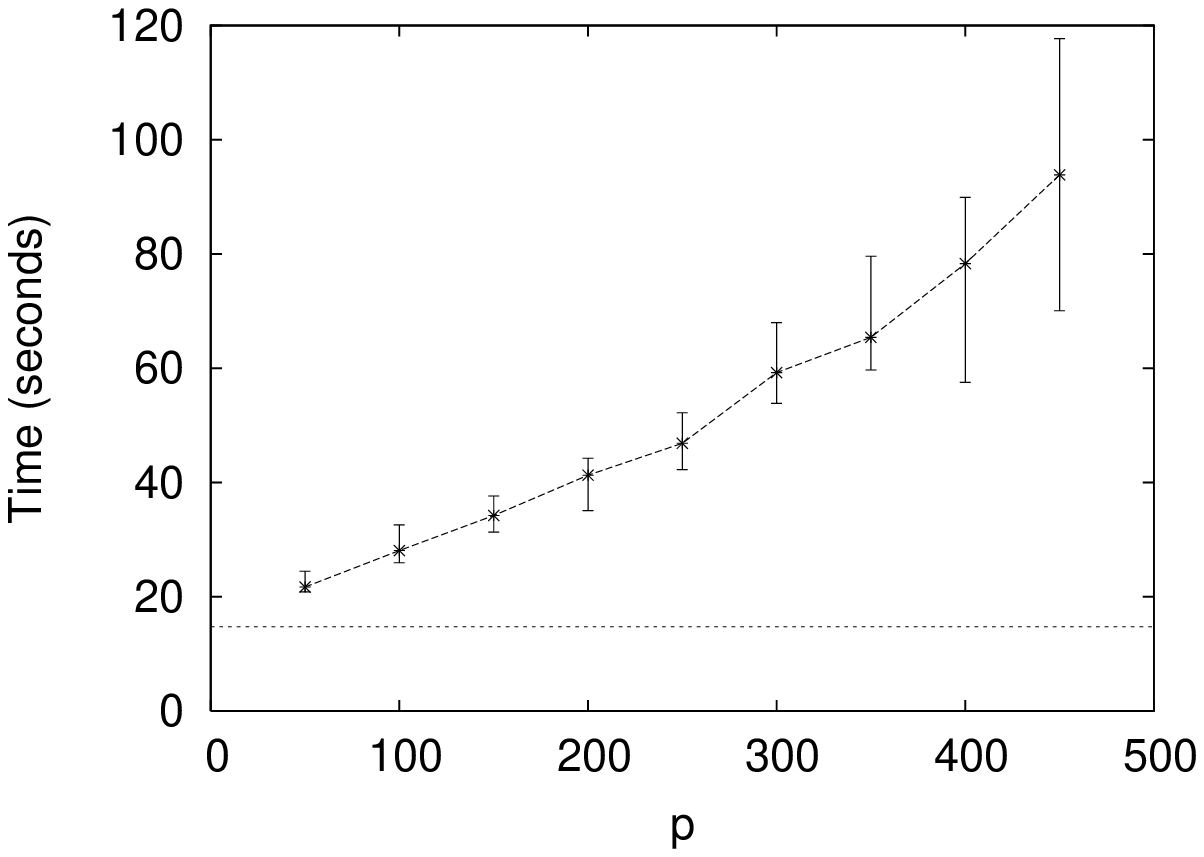}
\end{tabular}
\end{minipage}
\\
(d) $\smodels$ without lookahead: decisions (left), time in seconds (right)
\end{tabular}
\caption{Effects of adding randomly generated redundant rules to $\PHP$.}
\label{plot-rand}
\end{center}
\end{figure}

In Figure~\ref{plot-rand}, the median, minimum, and maximum number of
decisions and running times for the solvers on 
$\textsc{AddRandomRedundancy}(\PHP,n,p)$
are shown for $p=50,100,\ldots,450$  over 15 trials for each value of
$p$. 
The mean number of decisions (left) and running times (right) on the
original $\PHP$ are presented by the horizontal lines.
Notice that the number of added atoms and rules is linear to~$n$,
which is negligible 
to the number of atoms (in the order of  $n^2$) and rules ($n^3$) in $\PHP$.
For similar running times, the number of holes $n$ is $10$ for $\clasp$ and
$\smodels$ and $11$ for $\cmodels$. 
The results are very interesting: each of the solvers seems to react
individually to the added redundancy.
For $\cmodels$ (b), only a few added redundant rules are enough
to worsen its behavior.
For $\smodels$ (c), the number of decisions decreases linearly with
the number of added rules. 
However, the running times grow fast at the same time, most likely
due to $\smodels$' lookahead.
We also ran the experiment for $\smodels$ without using lookahead (d).
This had a visible effect on the number of decisions compared to
$\smodels$ on $\PHP$.

The most interesting effect is seen for $\clasp$ (a); $\clasp$
benefits from the added rules with respect to the number of decisions,
while the running times stay similar on the average, contrarily to the
other solvers. 
In addition to this robustness against redundancy, we believe that this
shows promise for further exploiting redundancy added in a controlled
way during search; the added rules give new possibilities to branch on
definitions which were not available in the original program.
However, for benefiting from redundancy with running times in mind,
optimized lightweight propagation mechanisms are essential.

As a final remark, an interesting observation is that the effect of the 
transformation presented in~\cite{DBLP:conf/ecai/AngerGJS06}, which
enables $\smodels$ to branch on the bodies of rules, having an
exponential effect on the proof complexity of a particular program
family, can be equivalently obtained by applying the ASP extension
rule.  
This may in part explain the effect of adding redundancy on the number
of decision made by $\smodels$.


\section{Conclusions}
We introduce  Extended ASP Tableaux, an extended tableau calculus for
normal logic programs under the stable model semantics.
We study the strength of the calculus, showing a tight
correspondence with Extended Resolution, which is among the most
powerful known propositional proof systems.
This sheds further light on the relation
of ASP and propositional satisfiability solving and their
underlying proof systems,  which 
we believe to be for the benefit of both of the communities.

Our experiments  show 
the intricate nature of the interplay between redundant problem
structure and the hardness of solving ASP instances.
We conjecture that more systematic use of the extension rule 
is possible and may even yield performance gains 
by considering in more detail the structural
properties of programs in particular problem domains.
One could also consider implementing branching on any possible formula
\emph{inside} a solver.
However, this would require novel heuristics, since choosing
the formula to branch on from the exponentially many alternatives is
nontrivial and is not applied in current solvers.
We find this an interesting future direction of research.
Another important research direction set forth by this
study is a more in-depth investigation into the effect of program
simplification on the hardness of solving ASP instances.


\section{Acknowledgements}

The authors thank Ilkka Niemelä for comments on a manuscript of this
article. Financial support from Helsinki Graduate School in Computer
Science and Engineering, Academy of Finland (grants \#211025 and
\#122399), Emil Aaltonen Foundation, Nokia Foundation, 
Finnish Foundation for Technology Promotion TES, Jenny and Antti
Wihuri Foundation (MJ), and Finnish Cultural Foundation (EO)
is gratefully acknowledged.


\end{document}